\documentclass{article}
\usepackage{geometry}
\usepackage{algpseudocode}
\usepackage[linesnumbered,lined,boxed,commentsnumbered,ruled,longend]{algorithm2e}

\usepackage{wrapfig}

\usepackage{times}
\usepackage{subcaption}

\usepackage[utf8]{inputenc} 
\usepackage[T1]{fontenc}    
\usepackage{hyperref}       
\usepackage{url}            
\usepackage{booktabs}       
\usepackage{amsfonts}       
\usepackage{nicefrac}       
\usepackage{microtype}      
\usepackage{xcolor}         

\usepackage{float}\usepackage{colortbl}
\usepackage{mathrsfs}
\usepackage{mathtools}
\usepackage{placeins}
\usepackage{xparse}
\usepackage{color}
\usepackage{comment}
\usepackage{natbib}

\usepackage{hyperref} 
    \hypersetup{
    	colorlinks = true,
    	linkcolor = blue,
    	anchorcolor = blue,
    	citecolor = blue,
    	filecolor = blue,
    	urlcolor = blue
    }

\usepackage{accents}

\usepackage{caption}
\usepackage{subcaption}
\usepackage{graphicx}

\usepackage{adjustbox}
\usepackage{booktabs}
\newcommand{\ra}[1]{\renewcommand{\arraystretch}{#1}}
\renewcommand{\arraystretch}{1.25}
\usepackage{lscape}

\usepackage{marginnote}
\usepackage{xcolor}

\definecolor{darkcyan}{rgb}{0.0, 0.55, 0.55}
\definecolor{MidnightBlue}{RGB}{25,25,112}
\definecolor{MidnightBlueComplementingGreen}{RGB}{25,112,25}
\definecolor{MidnightBlueComplementingPurple}{RGB}{112,25,112}
\definecolor{MidnightBlueComplementingRed}{RGB}{112,25,69}
\definecolor{WowColor}{rgb}{.75,0,.75}
\definecolor{MildlyAlarming}{rgb}{0.85,0.25,0.1}
\definecolor{SubtleColor}{rgb}{0,0,.50}
\definecolor{antiquefuchsia}{rgb}{0.57, 0.36, 0.51}
\definecolor{fashionfuchsia}{rgb}{0.96, 0.0, 0.63}
\definecolor{jade}{rgb}{0.0, 0.66, 0.42}
\definecolor{caribbeangreen}{rgb}{0.0, 0.8, 0.6}
\definecolor{aquamarine}{rgb}{0.5, 0.8, 0.85}
\definecolor{lightseagreen}{rgb}{0.13, 0.7, 0.67}
\definecolor{darkgreen}{rgb}{0.0, 0.2, 0.13}
\definecolor{darkspringgreen}{rgb}{0.09, 0.45, 0.27}
\definecolor{attentioncolor}{RGB}{152,90,81}
\definecolor{burgred}{RGB}{40,3,22}
\definecolor{AnnieGreen}{RGB}{17,123,92}
\definecolor{Turquoise}{RGB}{64,224,208}

\definecolor{darkjade}{RGB}{0,122,84}
\definecolor{Window1}{RGB}{92,150,31}%
    \definecolor{Window1dark}{RGB}{41,67,13}%
\definecolor{Window2}{RGB}{255,168,28}
    \definecolor{Window2dark}{RGB}{114,75,12}
\definecolor{Window3}{RGB}{255,96,33}
    \definecolor{Window3dark}{RGB}{97,36,12}
\definecolor{InputColor}{RGB}{20,255,177}
    \definecolor{InputColorlight}{RGB}{222,237,229}

\definecolor{RedAlizarin}{rgb}{0.82, 0.1, 0.26}

\usepackage{mathtools}
\usepackage{hyperref}
\makeatletter
\newcommand{\mytag}[2]{%
  \text{#1}%
  \@bsphack
  \begingroup
    \@onelevel@sanitize\@currentlabelname
    \edef\@currentlabelname{%
      \expandafter\strip@period\@currentlabelname\relax.\relax\@@@%
    }%
    \protected@write\@auxout{}{%
      \string\newlabel{#2}{%
        {#1}%
        {\thepage}%
        {\@currentlabelname}%
        {\@currentHref}{}%
      }%
    }%
  \endgroup
  \@esphack
}
\makeatother

\usepackage{soul}

\usepackage[colorinlistoftodos]{todonotes}

\NewDocumentCommand{\AK}{mo}{
    \IfValueF{#2}{
                        {{
                            \textcolor{magenta}{ 
                            \textbf{AK:}
                            \textit{{#1}}
                            }
                        }}
        }
    \IfValueT{#2}{
                        \marginnote{{\scriptsize
                            \textcolor{magenta}{ 
                            \textbf{AK:}
                            \textit{{#1}}
                            }
                        }}
        }
                    }

\NewDocumentCommand{\TF}{mo}{
    \IfValueF{#2}{
                        {{
                            \textcolor{blue}{ 
                            \textbf{Takashi:}
                            \textit{{#1}}
                            }
                        }}
        }
    \IfValueT{#2}{
                        \marginnote{{\scriptsize
                            \textcolor{blue}{ 
                            \textbf{Takashi:}
                            \textit{{#1}}
                            }
                        }}
        }
                    }

\NewDocumentCommand{\ocariz}{mo}{
    \IfValueF{#2}{
                        {{
                            \textcolor{purple}{ 
                            \textbf{Ocariz:}
                            \textit{{#1}}
                            }
                        }}
        }
    \IfValueT{#2}{
                        \marginnote{{\scriptsize
                            \textcolor{purple}{ 
                            \textbf{Ocariz:}
                            \textit{{#1}}
                            }
                        }}
        }
                    }

\usepackage{graphicx}

\usepackage[T1]{fontenc}    
\usepackage{booktabs}       

\usepackage{amsfonts,amsmath,amssymb,bbm}
\usepackage{enumitem}
\usepackage{graphicx}
\usepackage{adjustbox}
\usepackage[UKenglish]{isodate}
\usepackage{graphicx}
\usepackage[font=small,labelfont=it]{caption}
\usepackage{subcaption}
\usepackage{hyperref} 
    \hypersetup{
    	colorlinks = true,
    	linkcolor = blue,
    	anchorcolor = blue,
    	citecolor = blue,
    	filecolor = blue,
    	urlcolor = blue
    }
\usepackage{cleveref}
\usepackage{float}
\usepackage{enumitem}
\usepackage{multirow}
\usepackage{fancyvrb}
\usepackage{rotating}
    \usepackage{tikz}
    \usepackage{tikz-cd} 
    \pgfdeclarelayer{edgelayer}
    \pgfdeclarelayer{nodelayer}
    \pgfsetlayers{edgelayer,nodelayer,main}
    \tikzstyle{new style 0}=[fill={rgb,255: red,255; green,94; blue,247}, draw=black, shape=circle]
    \tikzstyle{pointy}=[fill=white, draw=black, shape=circle]
    \tikzstyle{pointy}=[->]
\usepackage{fontawesome}

\usepackage{lscape}

\renewcommand{\phi}{\varphi}

\newcommand{\rr}{\mathbb{R}}

\DeclareMathOperator*{\argmin}{arg\,min}

\newcommand{\N}{\mathbb{N}}

\newcommand{\inputdim}{n}

\newcommand{\prototype}{p}
\newcommand{\samplesize}{s}

\newcommand{\hatf}{\hat{f}}

\newcommand{\eqdef}{\ensuremath{\stackrel{\mbox{\upshape\tiny def.}}{=}}}

\NewDocumentCommand{\luca}{mo}{
    \IfValueF{#2}{
                        {{\scriptsize
                            \textcolor{green}{ 
                            \textbf{L:}
                            \textit{{#1}}
                            }
                        }}
        }
    \IfValueT{#2}{
                        \marginnote{{\scriptsize
                            \textcolor{green}{ 
                            \textbf{L:}
                            \textit{{#1}}
                            }
                        }}
        }
                    }
\NewDocumentCommand{\giulia}{mo}{
    \IfValueF{#2}{
                        {{\scriptsize
                            \textcolor{red}{ 
                            \textbf{GL:}
                            \textit{{#1}}
                            }
                        }}
        }
    \IfValueT{#2}{
                        \marginnote{{\scriptsize
                            \textcolor{red}{ 
                            \textbf{GL:}
                            \textit{{#1}}
                            }
                        }}
        }
}
\NewDocumentCommand{\anastasis}{mo}{
    \IfValueF{#2}{
                        {{\scriptsize
                            \textcolor{violet}{ 
                            \textbf{A:}
                            \textit{{#1}}
                            }
                        }}
        }
    \IfValueT{#2}{
                        \marginnote{{\scriptsize
                            \textcolor{violet}{ 
                            \textbf{A:}
                            \textit{{#1}}
                            }
                        }}
        }
                    }
                    
\NewDocumentCommand{\cody}{mo}{
    \IfValueF{#2}{
                        {{\scriptsize
                            \textcolor{orange}{ 
                            \textbf{A:}
                            \textit{{#1}}
                            }
                        }}
        }
    \IfValueT{#2}{
                        \marginnote{{\scriptsize
                            \textcolor{orange}{ 
                            \textbf{A:}
                            \textit{{#1}}
                            }
                        }}
        }
                    }

\NewDocumentCommand{\Greg}{mo}{
    \IfValueF{#2}{
                        {{\scriptsize
                            \textcolor{cyan}{ 
                            \textbf{Y:}
                            \textit{{#1}}
                            }
                        }}
        }
    \IfValueT{#2}{
                        \marginnote{{\scriptsize
                            \textcolor{cyan}{ 
                            \textbf{Y:}
                            \textit{{#1}}
                            }
                        }}
        }
                    }

\usepackage{appendix}
\usepackage{comment}

\NewDocumentCommand{\NN}{oo}{
    \ensuremath{
        \mathcal{NN}
        \IfValueT{#1}{_{#1}}\IfValueF{#1}{_{[d]}}
        \IfValueT{#2}{^{#2}}\IfValueF{#2}{^{\sigma}}
    }
}

\usepackage{amsthm}

\theoremstyle{plain}
\newtheorem{theorem}{Theorem}[section]
\newtheorem{proposition}[theorem]{Proposition}
\newtheorem{lemma}[theorem]{Lemma}

\theoremstyle{definition}
\newtheorem{definition}[theorem]{Definition}
\theoremstyle{remark}
\newtheorem{remark}[theorem]{Remark}

\newtheorem{example}{Example}

\usepackage{amsmath}
\usepackage{amssymb}
\usepackage{amsfonts}
\usepackage{mathtools}
\usepackage{amsthm}
\usepackage{algpseudocode}

\usepackage{hyperref}
    \hypersetup{
    	colorlinks = true,
    	linkcolor = blue,
    	anchorcolor = blue,
    	citecolor = blue,
    	filecolor = blue,
    	urlcolor = blue
}

\hypersetup{
pdftitle={Approximation Rates and VC-Bounds for Mixture of Experts},
pdfsubject={Approximation Rates and VC-Bounds for Mixture of Experts}
}

\begin{document}

\title{Approximation Rates and VC-Dimension Bounds for (P)ReLU MLP Mixture of Experts}

\author{Anastasis Kratsios\footnote{McMaster University and Vector Institute.  Email: \texttt{kratsioa@mcmaster.ca}},
Haitz S\'{a}ez de Oc\'{a}riz Borde\footnote{ University of Oxford.  Email: \texttt{chri6704@ox.ac.uk}},
Takashi Furuya\footnote{Shimane University.  Email: \texttt{takashi.furuya0101@gmail.com}},
Marc T. Law\footnote{NVIDIA.  Email: \texttt{marcl@nvidia.com}}
} 
\maketitle

\begin{abstract}%
Mixture-of-Experts (MoEs) can scale up beyond traditional deep learning models by employing a routing strategy in which each input is processed by a single ``expert'' deep learning model. This strategy allows us to scale up the number of parameters defining the MoE while maintaining sparse activation, i.e., MoEs only load a small number of their total parameters into GPU VRAM for the forward pass depending on the input. 
In this paper, we provide an approximation and learning-theoretic analysis of mixtures of expert MLPs with (P)ReLU activation functions.
We first prove that for every error level $\varepsilon>0$ and every Lipschitz function $f:[0,1]^n\to \mathbb{R}$, one can construct a MoMLP model (a Mixture-of-Experts comprising of (P)ReLU MLPs) which uniformly approximates $f$ to $\varepsilon$ accuracy over $[0,1]^n$, while only requiring networks of $\mathcal{O}(\varepsilon^{-1})$ parameters to be loaded in memory. Additionally, we show that MoMLPs can generalize since the entire MoMLP model has a (finite) VC dimension of $\tilde{O}(L\max\{nL,JW\})$, if there are $L$ experts and each expert has a depth and width of $J$ and $W$, respectively.  
\end{abstract}

\section{Introduction}
\label{s:Intro}

With the advent of large foundation models, the need to scale deep learning models beyond what is feasible to process on a single machine has become exceedingly relevant.  Mixture of Experts~(MoE) models present a solution to this problem via a sparse activation strategy. In MoEs, every input is first \textit{routed} to one amongst a large number of \textit{expert} deep learning models and then processed therewith.  This allows MoEs to scale up, while maintaining a low/constant computational cost during the forward pass since only a subset of the overall model needs to be loaded into GPU video random-access memory (VRAM) for any given input.
This has led to MoEs such as Mixtral \citep{jiang2024mixtral}, Gemini~\citep{Gemini}, and several others, e.g.~\citep{jacobs1991adaptive,jordan1995convergence,meila2000learning,shazeer2017outrageously, guu2020retrieval,lepikhin2021gshard, fedus2022switch, barham2022pathways,majid2024mixture}, to become a viable solution in scaling up large language models~\citep{radford2018improving,brown2020language}. However, the analytical and statistical underpinnings of MoEs in deep learning are comparatively less understood in contrast to empirical investigations.

This paper adds to the theoretical understanding of this subject by studying MoEs whose experts are (small) multilayer perceptrons (MLPs) with (P)ReLU activation function (MoMLPs).  A key feature of MoEs is that they can maintain a small/fixed computational cost during the forward pass, for any given input $x\in [0,1]^n$, even if the overall model complexity may be large.  Our main result (Theorem~\ref{thrm:Main}) analyzes the complexity of MoEs when uniformly approximating an arbitrary Lipschitz (Lebesgue) almost-everywhere continuously differentiable function $f:[0,1]^n\to \mathbb{R}^m$ by an MoMLP with (P)ReLU activation function to any prespecified error $\varepsilon>0$.  We focus on the trade-off between the maximum number of parameters loaded into VRAM by any expert model $\{\hat{f}_l\}_{l=1}^{\ell}$ while predicting from any given input, against the total number of experts required to maintain that constant number of activated parameters in the forward pass.  
Summarized in Table~\ref{tab:main_techincal__nodets}, our main result shows that a constant \textit{active complexity} in the forward pass can be maintained amongst all experts, but at the cost of an exponentially large number of locally-specialized experts $\{\hat{f}\}_{l=1}^{\ell}$ and regions of specialization $\{C_l\}_{l=1}^{\ell}$.  Our complexity estimates are approximately optimal as they nearly match the Vapnik-Chervonenkis (VC) lower bounds derived in~\cite{shen2021deep}.  That is, the uniform approximation of an arbitrary such $f$ on $[0,1]^\inputdim$, with an an error of $\varepsilon>0$, requires at least $\Omega(\varepsilon^{-\inputdim/2})$ \textit{total model parameters}~\citep{yarotsky2018optimal,kratsios2022universal,shen2021deep,shen2022optimal}.  It is here where MoEs have an advantage since not all of these parameters need to be loaded into active memory for any given input; thus MoEs are genuinely sparsely activated. 

From the statistical learning perspective, a key property of MoEs (e.g.\ the top MoMLP model) is that they can maintain a given level of activation in the forward pass while the entire MoE model can maintain a finite VC-dimension (Theorem~\ref{thrm:Bounded_VC__NeuralPathways}).  This is key, for instance, in classification applications, as the fundamental theorem of PAC learning (see e.g.~\citep[Theorem 6.7]{shalev2014understanding} or the results of~\citet{blumer1989learnability,Hanneke_2016_JMLR__CharacterizationLearnability,brukhim2022characterization}) implies that such a machine learning model generalizes beyond the training data if and only if it has finite VC dimension.

\paragraph{Summary of Contributions}

Table~\ref{tab:main_techincal__nodets} summarizes our main contributions, both to the approximation theory and learning theory of MoE models, in the context of the toy mixture of (P)ReLU MLP models.  All results illustrate the trade-off between individual (expert) complexity and the complexity shared across the set of experts when uniformly approximating a target $\alpha$-H\"{o}lder function, where $0<\alpha \le 1$.  
Note that $\alpha=1$ is equivalent to the target function being differentiable (Lebesgue) almost-everywhere.

Our \textit{approximation} theoretic result (Theorem~\ref{thrm:Main}) records the number of parameters required to perform a uniform approximation on a high-dimensional Euclidean space on $[0,1]^n$.  Our \textit{statistical learning} theoretic result (Theorem~\ref{thrm:Bounded_VC__NeuralPathways}) yields a bound on the VC dimension of the entire class of MoEs with just enough approximation power to perform this approximation.  

\begin{table}[H]
    \centering
    \caption{\textbf{Parametric Complexity and VC-Dimension} of MoMLP model with \textit{no. experts-to-expert complexity parameter} $r\in \mathbb{R}$; performing an $\varepsilon>0$ approximation of an $\alpha$-H\"{o}lder function $f:[0,1]^n \to \mathbb{R}$; for $n\in\mathbb{N}$.\\
    When $r\geq 0$ more model complexity is distributed across many ``small experts''.  When $r<0$ fewer experts define the MoE and, as a result, each expert MLP must depend on more parameters such that the entire MoE obtain an accurate approximation of the target function.  }
    \label{tab:main_techincal__nodets}
    \ra{1.3}
    \begin{tabular}{@{}ll@{}}
    \cmidrule[0.3ex](){1-2}
    \textbf{Parameter} & \textbf{Estimate} \\
    \midrule
    No.\ Expert Parameters &
    $\mathcal{O}(\max\{1,\varepsilon^{-r}\})$
    \\
    No.\ Experts
    &
    $
    \mathcal{O}\big(
        \max\{1,\varepsilon^{2r/n-1/\alpha}\}
    \big)
    $
    \\
    Routing Complexity
    &
    $
    \mathcal{O}\big(
        \max\{1,(\frac{2r}{n}-\frac{1}{\alpha})\log(\varepsilon)\}
    \big)
    $
    \\
    VC Dimension MoE & 
    $
        \tilde{\mathcal{O}}\big(
            \max\{1,\varepsilon^{2r/n-1/\alpha}\}
        \,
        \max\{
                \varepsilon^{2r/n-1/\alpha}
            ,
                \varepsilon^{-r}
        \}
    \big)
    $
    \\
    \arrayrulecolor{black}
    \bottomrule
    \end{tabular}
\end{table}
Observe that, setting $r=\frac{2}{n}$ in Table~\ref{tab:main_techincal__nodets}, yields for ReLU MLPs derived in~\citet{yarotsky2017error}; the optimality of which is expressed in terms of VC dimension in \citet{shen2022optimal}.  Likewise, the VC dimension of the MoMLP is roughly equal to that of ReLU MLPs computed in~\citet{bartlett2019nearly}.

\section{Related Work}
\label{s:Related_Research}
\textbf{Deep Learning Models with Few Parameters in Active Memory.}
Deep learning models with highly oscillatory ``super-expressive'' activation functions~\citep{yarotsky2020phase,yarotsky2021elementary,zhang2022deep} are known to achieve dimension-free approximation rates, thus effectively require a (relatively) feasible number of parameters to be loaded into VRAM during the forward pass.  As we will see in Proposition~\ref{prop:Unbounded_VC__SuperExpressive}, many of these models have an infinite VC dimension even when they are restricted to having a bounded depth and width; see~\citet[Lemma 3.1]{JiaoEtal_SIAMANalysisSINEXPBreakCOD_2023} for $\operatorname{ReLU-Sin-2^x}$-networks.  Their unbounded VC dimension implies that the classifiers implemented by these models
do not generalize on classification problems.   Thus, the real performance of these models does not need to achieve the approximation-theoretic optima since they can only learn from a finite number of noisy training instances. Alternatively, a feasible number of parameters in deep learning models with standard activation functions may be guaranteed by restricting classes of well-behaved target functions such as Barron functions~\citep{barron1993universal}, functions of mixed-smoothness \citep{suzuki2018adaptivity}, highly smooth functions~\citep{mhaskar1996neural,galimberti2022designing,gonon2023approximation,opschoor2022exponential}, convex functions~\citep{bach2017breaking}, functions with compositional structure~\citep{mhaskar2017and}, or other restricted classes. However, there are generally no verifiable guarantees that a target function encountered in practice has the necessary structure for these desired approximation theorems to hold.

\textbf{Universal Approximation in Deep Learning.} Several results have recently considered the expression power of deep learning models.  These include universal approximation guarantees for MLPs \citep{cybenko1989approximation,hornik1989multilayer,lu2017expressive,suzuki2018adaptivity,yarotsky2017error,yarotsky2018optimal,voigtlaender2019approximation,bolcskei2019optimal,guhring2020error,de2021approximation,devore2021neural,daubechies2022nonlinear,kratsios2022relu,zhang2022deep,opschoor2022exponential,zamanlooy2022learning,shen2022optimal,cuchiero2023global,voigtlaender2023universal,benth2023neural,mao2023rates,yang2024optimal}, CNNs~\citep{petersen2020equivalence,yarotsky2022universal}, spiking neural networks~\citep{MartinaSpiking2024}, residual neural networks~\citep{tabuada2021universal}, transformers~\citep{yun2019transformers,yun2020n,kratsios2022universal,fang2023attention}, random neural networks~\citep{gonon2023approximation}, recurrent neural network models~\citep{grigoryeva2018universal,gonon2021fading,hutter2022metric,galimberti2022designing,hoon2023minimal}, and several others. 
In each these cases, one typically considers the expressivity of a single ``expert'' model and not a mixture thereof.  Our analysis can be customized to any of these settings to yield analogues of Theorem~\ref{thrm:Main}.  

\textbf{Foundations of MoEs.}
MoE models have been heavily studied since their inception.  Most results have focused on identifying the correct expert to best route any given input to \citep{teicher1960mixture,teicher1963identifiability,wang1996mixed}, the construction of effective routing mechanisms~\citep{2017arXiv171109485W}
selection \citep{wang1996mixed}, MoE training~\citep{larochelle2009exploring,akbari2024alternating}, statistical convergence guarantees for classes of MoEs~\citep{chen1995optimal,ho2022convergence}, robustness guarantees for such models~\citep{puigcerver2022adversarial}, amongst several other types of guarantees.  However, to our knowledge, there are no available approximation guarantees for MoE or VC-dimension bounds for deep-learning-based MoEs.  Thus, our results would be adding to the approximation theoretic foundations of MoE models as well as to the statistical foundations of deep-learning-based MoEs.

\textbf{Prototypes and Partitioning.} Each region in our learned partition of the input space is associated with a \emph{representative point} therein called a \emph{prototype}. Prototypes (also called \emph{landmarks}) are routinely used in image classification~\citep{mensink2012metric}, few-shot learning~\citep{snell2017prototypical,Cao2020A}, dimensionality reduction~\citep{law2018dimensionality}, in complex networks~\citep{keller2023strain}, and geometric deep learning~\citep{ghadimi2021hyperbolic} to tractably encode massive structures. They are also standard in classical clustering algorithms such as $K$-medoids or $K$-means, wherein the part associated with each medoid (resp.\ centroid) defines a Voronoi cell or Voronoi region~\citep{Voronoi1908198287}. Moreover, while partitioning is commonly employed in deep learning for various purposes, such as proving universal approximation theorems~\citep{yarotsky2017error,LuShenZuoweiHaizhao_2021_DNNApprx_SmoothFunctions,guhring2021approximation} or facilitating clustering-based learning~\citep{zamanlooy2022learning,pmlr-v190-trask22a,ali_CA_2023_TensorizedTreeNNs_ApproximationTheory,srivastava2022expertnet}, existing approaches typically involve loading the entire model into VRAM. Our approach, however, differs by relying on a learned partition of the input space, where each part is associated with a distinct small neural network. Importantly, the complete set of networks forming the MoMLPs does not need to be simultaneously loaded into VRAM during training or inference.

\paragraph{Paper Overview.} Our paper is organized as follows.  Section~\ref{s:Prelim} contains preliminary notation, definitions, and mathematical background required for the formulation of our main results.  Section~\ref{s:Main_Result} contains our main approximation (Theorem~\ref{thrm:Main}) and learning theoretic (Theorem~\ref{thrm:Bounded_VC__NeuralPathways}) guarantees.  
Section~\ref{s:Proof} dives into the details of why \textit{MoEs can achieve arbitrary precision while maintaining a feasible active computational complexity} by explaining the derivation of our main approximation theorem; the details of which are relegated to Appendix~\ref{a:Proof}.  A technical version (Theorem~\ref{thm:Main__TechnicalVersion}) of our main approximation guarantee is then presented, which allows for the approximation of continuous functions of arbitrarily low regularity and for the organization of the experts defining the MoMLP into more general tree structures.
Appendix~\ref{a:Experiments} contains technical derivations of our main results as well as experimental elucidation of the benefit of MoEs, and specifically the toy MoMLP model.

\section{Preliminaries}
\label{s:Prelim}
We standardize our notation, define the necessary mathematical formalisms to state our main results and define our toy MoE Model.

\subsection{Notation.} 
We use the following notation: for any $f,g:\mathbb{R}\to \mathbb{R}$, we write $f\in \mathcal{O}(g)$ if there exist $x_0\in \mathbb{R}$ and $M\ge0$ such that for each $x\ge x_0$ we have $|f(x)|\le M g(x_0)$. 
Similarly, we write  $f\in \Omega(g)$ to denote the relation $g \in O(f)$.
The ReLU \textit{activation function} is given for every $x\in \mathbb{R}$ by $\text{ReLU}(x) = \max\{x,0\}$. 
For each $n\in \mathbb{N}_+$ and $C\subseteq \mathbb{R}^n$, the \textit{indicator function} $I_C$ of $C$ is defined by: for each $x\in \mathbb{R}^n$ set $I_C(x)=1$ if $x\in C$ and is $0$ otherwise.

\subsection{Background}
\paragraph{Multi-Layer Perceptrons}
We will consider MLPs with trainable PReLU activation functions.  
\begin{definition}[Trainable PReLU]
\label{defn:PReLU}
We define the trainable PReLU activation function $\sigma:\rr\times \rr\to \rr$ for each input $x\in \rr$ and each parameter $\alpha\in \rr$ as follows: 
\[
    \sigma_{\alpha} (x) 
    \eqdef \sigma (x,\alpha) 
    \eqdef 
    \begin{cases}
        x ~~~~~~~~~\text{ if } x\ge 0,\\
        \alpha x ~~~~~~\text{ otherwise.}
    \end{cases}
\]
\end{definition}
PReLU generalizes ReLU since $\text{ReLU}(x) = \sigma_{0} (x)$, and it makes the hyperparameter $\alpha$ of a Leaky ReLU learnable. 
We will often be applying our trainable activation functions component-wise. For positive integers $n,m$, we denote the set of $n\times m$ matrices by $\rr^{n\times m}$.  
More precisely, we mean the following operation defined for any $N\in \N$, $\bar{\alpha} \in \rr^{N}$ with $i^{th}$ entry denoted as $\bar{\alpha}_i$, and $x\in \rr^{N}$, by
\[
\sigma_{\bar{\alpha}}\bullet x
\eqdef
\left(
\sigma_{\bar{\alpha}_i}(x_i)
\right)_{i=1}^N
.
\]
We now define the class of multilayer perceptions (MLPs), with trainable activation functions. Fix $J\in\N$ and a multi-index $[d]\eqdef (d_0,\dots,d_{J+1})$, and let 
$P([d])=\sum_{j=0}^{J
}d_{j}(d_{j+1}+2) 
$.
We identify any vector $\theta \in \rr^{P([d])}$ with
\begin{equation}
\begin{aligned}
         \theta & \leftrightarrow
    \big(
        A^{(j)},b^{(j)},\bar{\alpha}^{(j)}
    \big)_{j=0}^J
\,\,\mbox{ and }\,\,
(A^{(j)},b^{(j)},\bar{\alpha}^{(j)}) & \in 
\rr^{d_{j+1}\times d_{j}}\times \rr^{d_{j}} \times \rr^{d_{j}}.
\end{aligned}
\end{equation}
We recursively define the representation function of a $[d]$-dimensional network by
\begin{equation}
    \begin{aligned}
   \rr^{P([d])}\times \rr^{d_0} \ni (\theta,x)
    & \mapsto 
    \hatf_{\theta}(x) 
        \eqdef 
    A^{(J)}
    \,
    x^{(J)}
   +
   b^{(J)}
   ,
    \\
    x^{(j+1)} &\eqdef 
    \sigma_{\bar{\alpha}^{(j)}}\bullet(
        A^{(j)} 
        \,
        x^{(j)}
            +
        b^{(j)})
    %
    \qquad \mbox{ for }j=0,\dots,J-1
    \\
    x^{(0)} &\eqdef x.
    \end{aligned}
    \label{eq_definition_ffNNrepresentation_function}
\end{equation}
We denote by $\NN[[d]]$ the family of $[d]$-dimensional \textit{multilayer perceptrons} (MLPs), $\{\hatf_{\theta}\}_{\theta \in \rr^{P([d])}}$ described by \eqref{eq_definition_ffNNrepresentation_function}.  
The subset of $\NN[[d]]$ consisting of networks $\hatf_{\theta}$ with each $\bar\alpha^{(j)}_i=(1,0)$  in Eq.~\eqref{eq_definition_ffNNrepresentation_function} is denoted by $\NN[[d]][\operatorname{ReLU}]$ and consists of the familiar deep ReLU MLPs. 
The set of ReLU MLPs with \textit{depth} $J$ and \textit{width} $W$ is denoted by $\NN[J,W:n,m]=\cup_{[d]}\, \NN[[d]]$, where the union is taken over all multi-indices $[d]=[d_0,\dots,d_{\tilde{J}}]$ with $n=d_0$, $m=d_{J+1}$, $d_0,\dots,d_{J+1}\le W$, and $\tilde{J}\le J$.

\paragraph{VC dimension} 
Let $\mathcal{F}$ be a set of functions from a subset $\mathcal{X}\subseteq \mathbb{R}^n$ to $\{0,1\}$; i.e.\ binary classifiers on $\mathcal{X}$.  The set $\mathcal{F}$ shatters (in the classical sense) a $k$-point subset $\{x_i\}_{i=1}^k\subseteq \mathcal{X}$ if $\mathcal{F}$ can represent every possible set of labels on those $k$-points; i.e.\ if $\#\,\{  (\hatf(x_i))_{i=1}^k \in \{0,1\}^k:\, \hatf\in \mathcal{F}\}=2^k$.  

As in~\citet{shen2022optimal}, we extend the definition of shattering from binary classifiers to real-valued functions as follows.  Let $\mathcal{F}$ be a set of functions from $[0,1]^n$ to $\mathbb{R}$.  The set $\mathcal{F}$ is said to shatter a $k$-point set $\{x_i\}_{i=1}^k\subseteq \mathcal{X}$ if
\begin{equation}
\label{eq:real_valued_shattering}
    \{ I_{(0,\infty)} \circ f:\,f\in \mathcal{F}\}
\end{equation}
shatters it, i.e.\ if all possible classifiers on $\{x_i\}_{i=1}^k$ are implementable in the sense that $\{ I_{(0,\infty)} \circ f:\,f\in \mathcal{F}\}=\{0,1\}^{\{x_i\}_{i=1}^k}$; here $I_{(0,\infty)}(t)=1$ if $t> 0$ and equals to $0$ otherwise. 
Denoted by $\operatorname{VC}(\mathcal{F})$, the \textit{VC dimension} of $\mathcal{F}$ is the cardinality of the largest $k$-point subset shattered by $\mathcal{F}$.  If $k$ is unbounded, then we say that $\mathcal{F}$ has an infinite VC dimension (over $\mathcal{X}$).  
One can show, see \citet{bartlett2019nearly}, that the VC-dimension of any such $\mathcal{F}$ is roughly the same as the pseudo-dimension of \citet{PollardEmpiricalProcesses_Book_1990} for a small modification of $\mathcal{F}$.

VC dimension measures the richness of a class of functions.  For example, in~\citet[Theorem 1]{harvey2017nearly}, the authors showed that the set of MLPs with $\operatorname{ReLU}$ activation function with $L\in \mathbb{N}_+$ layers, width and $W\in \mathbb{N}_+$ satisfying $W>O(L)>C^2$, where $C\ge 640$, satisfies
\begin{equation}
\label{eq:VCDim_MLPs}
        \operatorname{VC}(\NN[W,L])
    \in 
        \Omega\Big(
            WL\,\log_2(W/L)
        \Big)
.
\end{equation}
Nearly matching upper bounds are in \citet{bartlett2019nearly}.

We now define our toy mixture of experts model, the MoMLP.
\subsection{Definition: Our Toy Mixture of Experts Model - the MoMLP}

\begin{wrapfigure}{r}{0.40\textwidth} \vspace{-10pt}
\centering
\centerline{\includegraphics[width=0.96\linewidth]{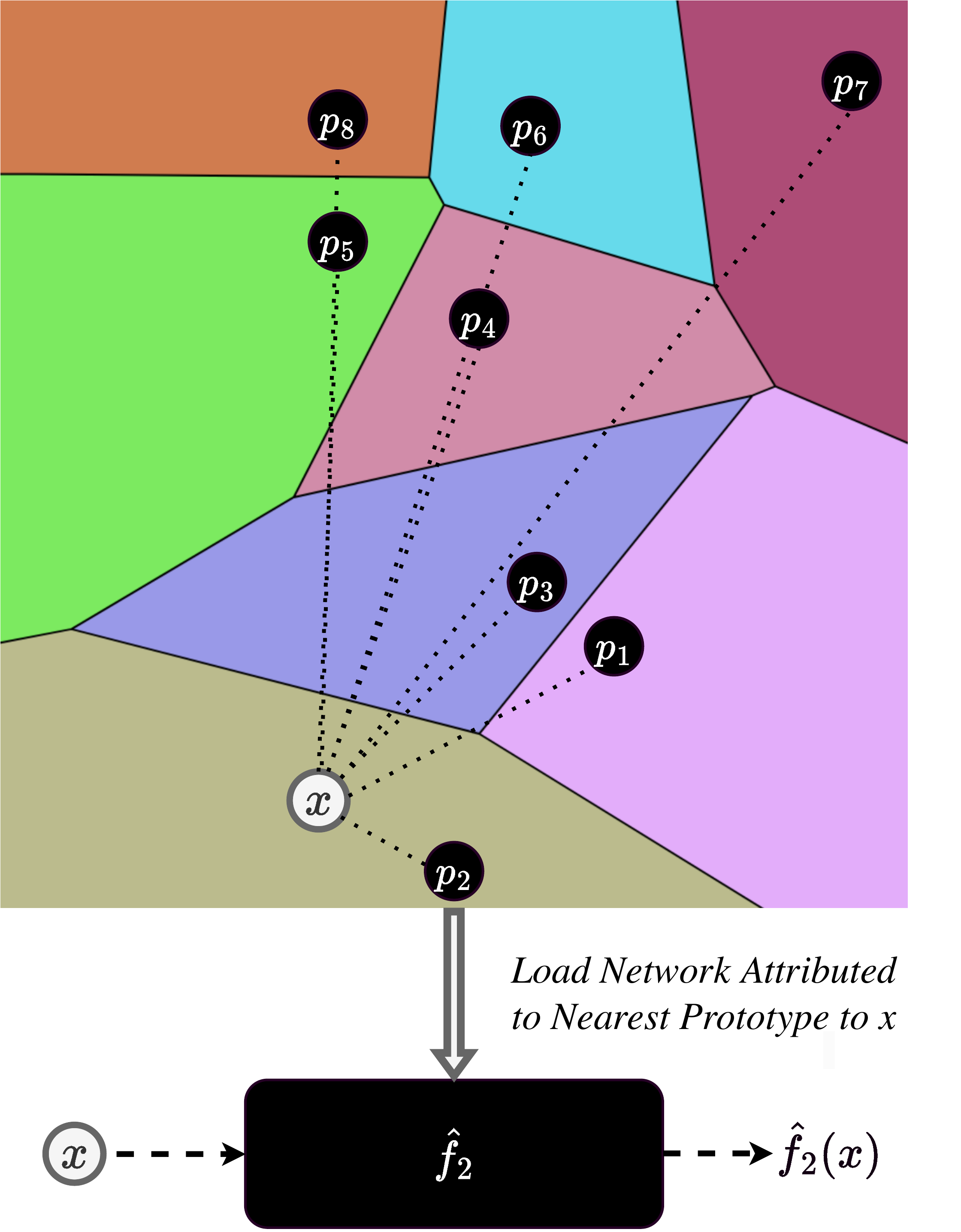}}\caption{
    $1$)~The distance from each input $x$ to all prototypes $\prototype_1,\dots,\prototype_8$ ($\ell=8$) is queried.~
    $2$)~The network ($\hatf_2$ in the figure) assigned to the nearest prototype ($\prototype_2$), is loaded onto the GPU and used for prediction.}
    \label{fig:neural_pathway__workflow} \vspace{-20pt}
\end{wrapfigure}

We study the following toy MoE model, where each expert (P)ReLU MLP specializes in a distinct region of the input domain $[0,1]^n$. 
Informally, these regions $C_1,\dots,C_{\ell}$ correspond to the sets of closest points (Voronoi cells) from a finite set of prototypes/landmarks $p_1,\dots,p_{\ell}$ in $[0,1]^n$, as illustrated in Figure~\ref{fig:neural_pathway__workflow}.  Associated to each region $C_i$ is a single expert MLP $\hat{f}_i$ with (P)ReLU activation function responsible for approximating the target function \textit{only thereon}.  
Here, the sparse gating procedure which \textit{routes} any given input $x\in [0,1]^n$ to the expert corresponding to the nearest prototype $p_i$ is implemented by a (finite) \textit{routing tree} $\mathcal{T}=(V,E)$ whose nodes $V$ are points in $[0,1]^n$ and leaves (terminal nodes) are the points $p_1,\dots,p_{\ell}$. 

We now formally define the classes of MoMLPs.   
\begin{definition}[MoMLPs]
\label{defn:NeuralPathways}
Let $J,W,{l},n \in \mathbb{N}$ and fix an activation function $\sigma\in C(\mathbb{R})$.  The set of MoMLPs with at-most $L$ leaves, depth $J$, and width $W$, denoted by $\mathcal{NP}_{J,W,{l}:n,m}^{\sigma}$, consists of all functions $\hatf:\mathbb{R}^n\to\mathbb{R}^m$ satisfying $\hatf = \sum_{i=1}^l\, f_i\,I_{C_i}$ 
where for $f_1,\dots,f_l\in \mathcal{NN}^{\operatorname{PReLU}}_{J,W}$, and \textit{distinct} prototypes $\prototype_1,\dots,\prototype_{{l}}\in \mathbb{R}^\inputdim$; whose induced Voronoi cells $\{C_i\}_{i=1}^{{l}}$ are defined recursively by
\begin{equation}
\label{eq:paritioning}
\begin{aligned}
    C_i & \eqdef \tilde{C}_i \setminus \bigcup_{j<i}\, \tilde{C_j},\\ 
\end{aligned}
\end{equation}
where for $i=1,\dots,l$ the (non-disjoint) cells are
\begin{equation}
    \tilde{C}_i \eqdef \big\{
        x\in [0,1]^n:\, \|x-p_i\|=\min_{j\in \{ 1,\dots,l \} }\, \|x-p_j\|
    \big\}.     
\end{equation}
\end{definition}

\paragraph{{The Routing Trees}}

The structure in any MoMLP is any tree with root node $\mathbb{R}^n$ and leaves given by the pairs $\{(p_i,f_i)\}_{i=1}^{l}$ or equivalently $\{(C_i,f_i)\}_{i=1}^{l}$.  The purpose of any such tree is simply to efficiently route an input $x\in \mathbb{R}^n$ to one of the $l$ ``leaf networks'' (the experts) $f_1,\dots,f_{l}$ using $\mathcal{O}(l\log(l))$ queries; to identify which Voronoi cells $\{C_i\}_{i=1}^{l}$ the point $x$ is contained in. 

\begin{remark}[{Partitioning in~\eqref{eq:paritioning} in classical computer science}]
The partitioning technique used to define Eq.~\eqref{eq:paritioning} is standard, see e.g.~\citet[Proof of Lemma 1.7; page 846]{KrauthgamerLeeMendelNaor_2005_GAFA__MEasuredDescentEmbedding}.  It is employed to ensure disjointness of the Voronoi cells; this guarantees that no input is assigned to multiple prototypes.  
To keep notation tidy, we use $\mathcal{NP}_{J,W,L}^{\sigma}$ (resp.\ $\mathcal{N}^{\sigma}_{W,L}$), to abbreviate $\mathcal{NP}_{J,W,L:n,m}^{\sigma}$ (resp.\ $\mathcal{N}^{\sigma}_{W,L:n,m}$) whenever $n$ and $m$ are clear from the context.
\end{remark}

\section{Main Results}
\label{s:Main_Result}
We first present our main approximation theoretic guarantee, which gives complexity estimates for mixtures of MLPs with trainable PReLU activation functions when uniformly approximating arbitrary locally-H\"{o}lder function on the closed unit ball of $\rr^n$, defined by $\overline{B}_n(0,1)\eqdef \{x\in \rr^n:\,\|x\|\le 1\}$.

Our rates depend on a ``number of experts-to-expert complexity trade-off parameter'' $r\in \mathbb{R}$ which determines how fast the overall MoE complexity scales, in terms of the number of experts and the complexity of each expert, as the approximation error becomes small.
Setting $r<0$ implies that more model complexity will be loaded into each expert MLP and there will be fewer experts defining the MoE.  In contrast, setting $r>0$ loads less complexity in each expert MLP at the cost of more experts in the MoE.  In particular, when $r=0$, each expert will have constant complexity even when the approximation error becomes arbitrarily small.

\begin{theorem}[{Trade-Off: No.\ Expert vs.\ Expert Complexity}]
\label{thrm:Main}
Suppose that $\sigma$ satisfies Definition~\ref{defn:PReLU}.
Fix an ``number of experts-to-expert complexity trade-off parameter'' $r\in \mathbb{R}$.  
\hfill\\
For every $\alpha$-H\"{o}lder map $f: \overline{B}_n(0,1) \to \rr^m$ with $0<\alpha\le 1$ and each approximation error $\varepsilon>0$, there is a $p\in \mathbb{N}_+$, a binary tree $\mathcal{T}\eqdef(V,E)$ with leaves $\mathcal{L}\eqdef \{(v_i,\theta_i)\}_{i=1}^L \subseteq  {\mathcal{K}} \times\rr^p$ and a family of MLPs with (P)ReLU activation function $\{\hat{f}_{\theta_i}\}_{i=1}^L$ defined by $p$ parameters and mapping $\mathbb{R}^n$ to $\mathbb{R}^m$ satisfying:
\begin{equation*}
    \max_{x\in  {\mathcal{K}} }\,
        \min_{(v_i,\theta_i)\in \mathcal{L}}\,
            \|
                    x
                -
                    v_i
            \|
        \in 
            \Theta\big(
                \varepsilon^{1/\alpha - 2r/n}
            \big)
\end{equation*}
and for each $x\in  {\mathcal{K}} $ and $i=1,\dots,L$, if $\|x-v_i\|<\delta$ then
\[
    \|
            f(x)
        -
            f_{\theta_i}(x)
    \|
    <
        \varepsilon
    .
\]
The depth and width of each $\hatf_{\theta_i}$ and the number of leaves, height, and number of nodes required to build the $\nu$-ary tree are all recorded in Table~\ref{tab:main_techincal__nodets}.
\end{theorem}

Next, we demonstrate that the MoMLP model can generalize and generate functions that are PAC-learnable, thanks to its finite VC dimension. This property, however, breaks down in MLP models with super-expressive activation functions.  
\begin{theorem}[VC-Dimension Bounds for MoMLPs - MoMLPs Can Generalize]
\label{thrm:Bounded_VC__NeuralPathways}
Let $J,W,L,n\in \mathbb{N}_+$. Then $\operatorname{VC}\big(\mathcal{NN}_{J,W,L:n,1}^{\operatorname{PReLU}}\big)$ is of
\begin{equation}
        \mathcal{O}\big(
            L\log(L)^2
            \,
            \max\{
                    nL\log(L)
                ,
                    JW^2\log(JW)      
            \}
        \big)    
\end{equation}
In particular, $\mathrm{VC}(\mathcal{NP}_{J,W,L:n,1}^{ReLU})< \infty$.
\end{theorem}

\subsection{Discussion}
\label{s:Discussion}

\textbf{Trade-off between Number of Experts and Expert Complexity.} Our results suggest that, theoretically, successful MoE models may not need each expert to be highly overparameterized if there are enough experts.  This hypothesis is ablated experimentally in Appendix~\ref{a:Experiments} in the context of irregular function approximation in low-dimension space; which is equivalent to high-dimensional regular function approximation (see Appendix~\ref{a:COI} for a discussion on this later point).

\textbf{Pruning.} Additionally, one might consider the option of pruning a sizable model, conceivably trained on a GPU with a larger VRAM, for utilization on a smaller GPU during inference as an alternative to our method. Nevertheless, in frameworks like PyTorch, pruning does not result in a reduction of operations, acceleration, or diminished VRAM memory usage. Instead, pruning only masks the original model weights with zeros. The reduction in model size occurs only when saved in offline memory in sparse mode, which, in any case, is not a significant concern.

\textbf{Logarithmic number of queries via trees.} 
For many prototypes, as in our main guarantee, the MoMLPs only need to evaluate the distance between the given input and a logarithmic number of prototypes—specifically, one for each level in the tree—when using deep binary trees to hierarchically refine the Voronoi cells. Thus, a given machine never processes the exponential number of prototypes, and only $\nu\lceil \log_{\nu}(K)\rceil$ prototypes are ever queried for any given input; when trees are $\nu$-ary (as in Theorem~\ref{thm:Main__TechnicalVersion}), and where $K$ denotes the number of prototypes. 
Since we consider that prototypes are queried separately and before loading MoMLPs, we do not take them into account when counting the number of learnable parameters.
Moreover, the size of our prototypes is negligible in our experiments.

\subsection{Application: Controlling The Complexity in VRAM maintaining a Finite VC Dimension}

\textbf{Super-Expressive Activation Functions Have Infinite VC-Dimension.} We complement the main result of \citet{bartlett2019nearly} by demonstrating that the class of unstable MLPs \citep{shen2022deep_JMLR} possesses infinite VC dimension, even when they have finite depth and width. Thus, while they may serve as a gold standard from the perspective of approximation theory, they should not be considered a benchmark gold standard from the viewpoint of learning theory.

We consider a mild extension of the super-expressive activation function of \citet{shen2022deep_JMLR}.  This parametric extension allows it to implement the identity map on the real line as well as the original super-expressive activation function thereof.
\begin{definition}[{Trainable Super-Expressive Activation Function}]
\label{defn_TrainableActivation_SuperExpressive}
A trainable action function $\sigma:\rr\times \rr\to \rr$ is of super-expressive type if for all $\alpha\in \rr$
\[
    \sigma_{\alpha}:
\rr
    \ni x
\mapsto 
        \alpha x 
        +
        (1-\alpha)
        \sigma^{\star}(x)
\]
where $\sigma^{\star}:\rr\to\rr$ is given by
\begin{equation}
\label{eq:activation_superexpressive__nontrainable}
        \sigma^{\star}(x)
    \eqdef 
        \begin{cases}
            |x \mod{2}| & \mbox{ for } 0\le x\\
            \frac{x}{|x|+1} & \mbox{ for } x < 0
    .
        \end{cases}
\end{equation}
\end{definition}

\begin{proposition}[MLPs with Super-Expressive Activation Do Not Generalize]
\label{prop:Unbounded_VC__SuperExpressive}
Let $\mathcal{F}$ denote the set of MLPs with activation function in Definition~\ref{defn_TrainableActivation_SuperExpressive}, depth at-most $15$, and width at-most $36n(2n+2)$.  Then $\operatorname{VC}(\mathcal{F})=\infty$.
\end{proposition}

The VC dimension bounds for the standard MLP model, MLP with a super-expressive activation function as proposed by \citet{shen2022deep_JMLR}, and the MoMLP model are summarized in Table~\ref{tab:summary_VCDim}.

\vspace{-0.25cm}
\begin{table}[H]
\caption{VC Dimension of the  MoMLPs, ReLU MLP, and MLP model with Super-Expressive Activation function of \citet{shen2022deep_JMLR}.  All models have depth $J$, width $W$, and (when applicable) $L$ leaves; where $J,W,L,n\in \mathbb{N}_+$.
}
\label{tab:summary_VCDim}
\ra{1.3}
    \centering
    \begin{tabular}{@{}lll@{}}
        \cmidrule[0.3ex](){1-3}
        \textbf{Model} &  \textbf{VC Dim} & \textbf{Ref.}\\
            \midrule
            MoMLPs
            & 
            $
                \mathcal{O}\big(
                    L\log(L)^2
                    \,
                    \max\{
                            nL\log(L)
                        ,
                            JW^2\log(JW)      
                    \}
                \big)
            $
            & 
            Thrm~\ref{thrm:Bounded_VC__NeuralPathways}
    \\
    ReLU MLP
    &
    $
        \mathcal{O}\big(
            JW^2\log(JW)      
        \big)
    $
    &
    \citet{bartlett2019nearly}
    \\
    Super-Expressive
    &
    $\infty$
    &
    Prop~\ref{prop:Unbounded_VC__SuperExpressive}
    \\
    \arrayrulecolor{black}
    \bottomrule
    \end{tabular}
\end{table}

\section{Overview of Derivation}
\label{s:Proof}
We now overview the proof of our main result and its full technical formulation.  
These objectives require us to recall definitions from the analysis of metric spaces, which were not required in the statement of our main result but which are required when overviewing our proof.  

\subsection{Technical Definitions}
\label{s:Proof__ss:PrelimDefs}
The metric ball in $(\mathcal X,d)$ of radius $r > 0$ at $x \in \mathcal X$ is denoted by $\operatorname{Ball}_{(\mathcal X,d)}(x,r) \eqdef \{ z \in \mathcal X \colon d(x,z) < r \}$.
A metric space $(\mathcal X,d)$ is called \textit{doubling}, if there is $C\in \N_+$ for which every metric ball in $(\mathcal X,d)$ can be covered by at most $C$ metric balls of half its radius.  
The smallest such constant is called $(\mathcal X,d)$'s \textit{doubling number}, and is here denoted by $C_{(\mathcal X,d)}$.  Though this definition may seem abstract at first, \citet[Theorem 12.1]{heinonen2001lectures} provides an almost familiar characterization of all doubling metric spaces; indeed,  $\mathcal{K}$ is a doubling metric space if and only if it can be identified via a suitable invertible map\footnote{So called quasi-symmetric maps, see \citep[page 78]{heinonen2001lectures}.} with a subset of some Euclidean space.  Every subset of $\rr^n$, for any $n\in \N_+$, is a doubling metric space; see \citet[Chapter 9]{RobinsonAttactorsDimensionEmbeddings_2011_Book} for details.  

\begin{example}[Subsets of Euclidean Spaces]
\label{ex:doubling_Ball__Euclidean}
Fix a dimension $n\in \mathbb{N}_+$.  
The doubling number of any subset of Euclidean space is\footnote{See~\citet[Lemma 9.2]{RobinsonAttactorsDimensionEmbeddings_2011_Book} and the brief computations made in the proof of \citet[Lemma 9.4]{RobinsonAttactorsDimensionEmbeddings_2011_Book}. }%
 at most $2^{n+1}$.
\end{example}

In what follows, all \textit{logarithms} will be taken \textit{base} $2$, unless explicitly stated otherwise, i.e.\ $\log_v$ is base $v$ for a given $v\in \mathbb{N}_+$ and $\log\eqdef \log_2$.  As in \citet[page 762]{PetrovaWojtaszczyk_2023_LipWidths_CA}, the radius of a subset $A\subseteq \rr^n$, denoted by $\operatorname{rad}(A)$, is defined by
\begin{equation}
        \operatorname{rad}(A)
    \eqdef 
        \inf_{x \in \rr^n}\, \sup_{a\in A}\,
            \|x-a\|
    .    
\end{equation}
The diameter of any such set $A$, denoted by $\operatorname{diam}(A)$, satisfies the sharp inequality $\operatorname{diam}(A)\le 2\operatorname{rad}(A)$.

Finally, let us recall the notion of a uniformly continuous function.  
Fix $n,m\in \mathbb{N}_+$ and let $X\subset \mathbb{R}^n$.  Let $\omega:[0,\infty)\to[0,\infty)$ be a monotonically increasing function which is continuous at $0$ and satisfies $\omega(0)=0$.  Such an $\omega$ is called a \textit{modulus of continuity}.  
A function $f:X\rightarrow \mathbb{R}^m$ is said to be $\omega$-uniformly continuous if
\[
        \|
            f(x)-f(y)
        \|
    \le
        \omega\big(
            \|
                x - y 
            \|
        \big)
\]
holds for all $x,y\in \mathcal{X}$.  We note that every continuous function is uniformly continuous if $\mathcal{X}$ is compact and that its modulus of continuity may depend on $\mathcal{X}$.  Furthermore, we note that every $(\alpha,L)$-H\"{o}lder function is uniformly continuous with modulus of continuity $\omega(t)=L\,t^{\alpha}$.

\subsection{Helping to Explain MoEs via Proof Sketch}
\label{s:Proof__ss:ProofSketch}

\begin{lemma}[Size of a Tree Whose Nodes Form a $\delta$-net of a Compact Subset of $\mathbb{R}^n$]
\label{lem:complete_n_ary_tree_counting}
Let $\mathcal{K}$ be a compact subset of $\rr^n$ whose doubling number is $C$.  
Fix $v\in \mathbb{N}$ with $v\ge 2$, and $0<\delta \le \operatorname{rad}(\mathcal{K})$. 
There exists an $v$-ary tree $\mathcal{T}\eqdef(V,E)$ with leaves $\mathcal{L}\subseteq \mathcal{K}$ satisfying
\begin{equation}
\label{eq:packing_condition}
    \max_{x\in \mathcal{K}}\,
        \min_{v\in \mathcal{L}}\,
            \|
                    x
                -
                    v
            \|
        <
            \delta
.
\end{equation}
Furthermore, the number of leaves $L\eqdef \#\mathcal{L}$, height, and total number of nodes $\#V$ of the tree $\mathcal{T}$ are 
\begin{itemize}
    \item[(i)] \textbf{Leaves}: at most $
        L 
    =
        v^{
            \big\lceil
                c\,\log(C)\big( 
                            1 
                        + 
                            \log(\delta^{-1}\operatorname{diam}( {\mathcal{K}} ))
                    \big)
            \big\rceil
        ,
        }$
    \item[(ii)] \textbf{Height}: $\big\lceil
                c\,\log(C)\big( 
                            1 
                        + 
                            \log(\delta^{-1}\operatorname{diam}( {\mathcal{K}} ))
                    \big)
            \big\rceil$,
    \item[(iii)] \textbf{Nodes}: At most $
        \frac{v^{
            \lceil
                c\,\log(C)( 
                            1 
                        + 
                            \log(\delta^{-1}\operatorname{diam}( {\mathcal{K}} ))
                    )
            \rceil + 1
        } - 1}{
        \big\lceil
                c\log(C)\big( 
                            1 
                        + 
                            \log(\delta^{-1}\operatorname{diam}( {\mathcal{K}} ))
                    \big)
            \big\rceil - 1
        }
    $
\end{itemize}
where $c\eqdef 1/\log(v)$.  
\end{lemma}

At each node of the tree, we will place an MLP which only locally approximates the target function on a little ball of suitable radius (implied by the tree valency $v$ and height $h$) of lemma~\ref{lem:complete_n_ary_tree_counting}.  I.e.\ by the storage space we would like to allocate to our MoMLP model.
The next step of the proof relies on a mild extension of the \textit{quantitative universal approximation theorem} in \citet{shen2022deep_JMLR,LuShenYangZhang_SIAm_2021_OptimalApproxSmooth} to the multivariate case, as well as an extension of the multivariate approximation result of \citet[Proposition 3.10]{acciaio2023designing} beyond the H\"{o}lder case.  

\begin{lemma}[Vector-Valued Universal Approximation Theorem {with Explicit Diameter Dependence}]
\label{lem_universal_approximation_improved_rates}
Let $n,m\in \N_+$ with $n\ge 3$, ${\mathcal{K}}\subseteq \rr^n$ be compact set {of radius $\delta\ge 0$}, $f:{\mathcal{K}}\to \rr^m$ be uniformly continuous with strictly monotone continuous modulus of continuity $\omega$.  
Let $\sigma$ be an activation function as in Definitions~\ref{defn_TrainableActivation_SuperExpressive} or~\ref{defn:PReLU}. 
For each $\varepsilon>0$, there exists an MLP $\hatf_{\theta}:\rr^n\to \rr^m$ with trainable activation function $\sigma$ satisfying the uniform estimate
\[
\sup_{x \in {\mathcal{K}} }\, 
    \|
        f(x) - \hatf_{\theta}(x)
    \|
\le
    \epsilon.
\]
The depth and width of $\hatf$ are recorded in Table~\ref{tab_spacecomplexity_universalFFNN}.  
\end{lemma}

\begin{table}[H]
    \centering
    \caption{{Complexity of the MLP $\hatf_{\theta}$ in Lemma~\ref{lem_universal_approximation_improved_rates}.  See Table~\ref{tab_spacecomplexity_universalFFNN__detailed} in Appendix~\ref{a:Tables} for more detailed estimates.}}
    \label{tab_spacecomplexity_universalFFNN}
    \ra{1.3}
    \begin{tabular}{@{}lll@{}}
    \cmidrule[0.3ex](){1-3}
    Activation $\sigma$ & Super Expressive~\ref{defn_TrainableActivation_SuperExpressive}
    &
    PReLU~\ref{defn:PReLU}
    \\
    \midrule
    Depth $(J)$ &
    $ 
    \mathcal{O}(1)
    $
    & 
    $
        \mathcal{O}\big(
            (
                \delta/\omega^{-1}(\varepsilon)
            )^{n/2}
        \big)
    $\\
    Width ($\max_{j}\,d_j$)
    &
    $
         \mathcal{O}(1)
    $
    &
    $
        \mathcal{O}(1)
    $
    \\
    \arrayrulecolor{black}
    \bottomrule
    \end{tabular}
\end{table}

Combining Lemmata~\ref{lem:complete_n_ary_tree_counting} and~\ref{lem_universal_approximation_improved_rates} we obtain Theorem~\ref{thrm:Main}.
We now present the technical version of Theorem~\ref{thrm:Main}.  This result allows distributed neural computing using $\nu$-ary trees and allows for the approximation general uniformly continuous target functions.

\begin{theorem}[{Trade-Off: No.\ Expert vs.\ Expert Complexity - Technical Version of Theorem~\ref{thrm:Main}}]
\label{thm:Main__TechnicalVersion}
Suppose that $\sigma$ satisfies Definition~\ref{defn:PReLU}.
Let $ {\mathcal{K}} $ be a compact subset of $\rr^n$ whose doubling number is $C$.
Fix an ``number of experts-to-expert complexity trade-off parameter'' $r\in \mathbb{R}$ and a ``valency parameter'' $\nu\in \mathbb{N}$ with $\nu\ge 2$.  
\hfill\\
For every uniformly continuous map $f: {\mathcal{K}} \to \rr^m$ with modulus of continuity $\omega$ and each approximation error $\varepsilon>0$, $p\in \mathbb{N}_+$, there is an $\nu$-ary tree $\mathcal{T}\eqdef(V,E)$ with leaves $\mathcal{L}\eqdef \{(v_i,\theta_i)\}_{i=1}^L \subseteq  {\mathcal{K}} \times\rr^p$ and a family of MLPs with (P)ReLU activation function $\{\hat{f}_{\theta_i}\}_{i=1}^L$ defined by $p$ parameters mapping $\mathbb{R}^n$ to $\mathbb{R}^m$ satisfying:
\begin{equation*}
    \max_{x\in  {\mathcal{K}} }\,
        \min_{(v_i,\theta_i)\in \mathcal{L}}\,
            \|
                    x
                -
                    v_i
            \|
        <
            \frac{\varepsilon^{-2r/n}}{2}
                \,
                \omega^{-1}\biggl(
                    \frac{\varepsilon}{
                    131 \, (nm)^{1/2}
                    }
                \biggr)  
\end{equation*}
and for each $x\in  {\mathcal{K}} $ and $i=1,\dots,L$, if $\|x-v_i\|<\delta$ then
\[
    \|
            f(x)
        -
            f_{\theta_i}(x)
    \|
    <
        \varepsilon
    .
\]
The depth and width of each $\hatf_{\theta_i}$ and the number of leaves, height, and number of nodes required to build the $\nu$-ary tree are all recorded in Table~\ref{tab:main_techincal__nodets}.
\end{theorem}

\begin{table}[H]
    \centering
    \caption{Complexity of Feedforward Neural Network $\hatf_{\theta}$ and the $\nu$-ary {routing} tree in Theorem~\ref{thm:Main__TechnicalVersion}.  
    Here $c\eqdef \log(v)^{-1}$.}
    \label{tab:main}
    \ra{1.3}
    \begin{tabular}{@{}ll@{}}
    \cmidrule[0.3ex](){1-2}
    \textbf{Parameter} & \textbf{Estimate} \\
    \midrule
    Depth $(J)$ 
    & 
    $
        \mathcal{O}(\max\{1,\varepsilon^{-r}\})
    $
    \\
    Width ($\max_{j}\,d_j$)
    &
    $
        \mathcal{O}(1)
    $
    \\
    No. Experts
    &
    $
    \mathcal{O}\big(
       \max\{1, \varepsilon^{2r/n}/\omega^{-1}(\varepsilon)\}
    \big)
    $
    \\
    Routing Complexity
    &
    $
    \mathcal{O}\big(
        \max\{
                1
            ,
                \log(\varepsilon^{2r/n}/\omega^{-1}(\varepsilon))
        \}
    \big)
    $
    \\
    \arrayrulecolor{black}
    \bottomrule
    \end{tabular}
    \caption*{See Table~\ref{tab:main_techincal} in Appendix~\ref{a:Tables} for more detailed estimates.}
\end{table}

\section{Conclusion}
\label{s:Conclusion}

We presented approximation-theoretic and statistical foundations for MoEs by analysing the MoMLP model.  We found that MoMLPs can indeed achieve arbitrary uniform approximation accuracy of continuous functions on compact subsets of Euclidean space while maintaining a feasible number of parameters in active GPU VRAM memory (Theorem~\ref{thm:Main__TechnicalVersion}).  However, this naturally comes at the cost of requiring an exponential number of experts. We also obtain upper bounds on the VC dimension of the MoMLP model (Theorem~\ref{thrm:Bounded_VC__NeuralPathways}), akin to the results of \cite{bartlett2019nearly} for ReLU MLPs, showing that deep-learning-based MoEs can indeed generalize.

\section*{Limitations and Future Work}
This paper provides theoretical support, both in terms of approximation theory (approximation rates) and learning theory (VC dimension bounds), of MoE models.  As such there is no clear limitation to our analysis.  In future work, one can use the tools developed here to similarly study MoEs of other (deep) learning models.

\section*{Societal Impact}
This paper presents work aimed at advancing the theoretical foundations of machine learning. Therefore, the majority of societal consequences are more positive than negative, as they acknowledge the necessity for principled machine learning and contribute to the development of understanding the capabilities of AI.

    \section*{Acknowledgment and Funding}
    \label{s:ackn}
    The authors thank James Lucas, Rafid Mahmood, and Gabriel Conant for their helpful feedback while preparing the manuscript.  
    
    AK acknowledges financial support from an NSERC Discovery Grant No.\ RGPIN-2023-04482 and their McMaster Startup Funds.  AK also acknowledges that resources used in preparing this research were provided, in part, by the Province of Ontario, the Government of Canada through CIFAR, and companies sponsoring the Vector Institute~\href{https://vectorinstitute.ai/partnerships/current-partners/}{https://vectorinstitute.ai/partnerships/current-partners/}.
    HSOB acknowledges financial support from an EV travel grant.

\newpage
\appendix

\section{Detailed Tables And Rates}
\label{a:Tables}

\begin{table}[H]
    \centering
    \caption{{Complexity of the MLP $\hatf_{\theta}$ in Lemma~\ref{lem_universal_approximation_improved_rates}.}}
    \label{tab_spacecomplexity_universalFFNN__detailed}
    \ra{1.3}
    \begin{tabular}{@{}lll@{}}
    \cmidrule[0.3ex](){1-3}
    Activation $\sigma$ & Super Expressive~\ref{defn_TrainableActivation_SuperExpressive}
    &
    PReLU~\ref{defn:PReLU}
    \\
    \midrule
    Depth $(J)$ &
    $ 15 m$
    & 
    $
        m\left(
            19 + 2n  + 
            11
            \biggl\lceil
                \Big(
                    \frac{
                    {\delta}\,
                    2^{3/2}
                    n^{1/2}
                    }{
                    (n+1)^{1/2}
                    \omega^{-1}\big(
                        \varepsilon/(131 \sqrt{n\,m})
                    \big)
                    }
                \Big)^{n/2}
            \biggr\rceil
        \right)
    $\\
    Width ($\max_{j}\,d_j$)
    &
    $
         36 n (2n +1 ) +m  
    $
    &
    $
        16 
        \max\{n, 3\}
        +
        m
    $
    \\
    \arrayrulecolor{black}
    \bottomrule
    \end{tabular}
\end{table}

\begin{table}[H]
    \centering
    \caption{Complexity of Feedforward Neural Network $\hatf_{\theta}$ and the $\nu$-ary tree in Theorem~\ref{thm:Main__TechnicalVersion}.  
    Here $c\eqdef \log(v)^{-1}$.}
    \label{tab:main_techincal}
    \ra{1.3}
    \begin{tabular}{@{}ll@{}}
    \cmidrule[0.3ex](){1-2}
    \textbf{Parameter} & \textbf{Estimate} \\
    \midrule
    Depth $(J)$ 
    & 
    $
        m\left(
            19 + 2n  + 
            11
            \lceil
            \varepsilon^{-r}
            \rceil
        \right)
    $
    \\
    Width ($\max_{j}\,d_j$)
    &
    $
        16 \max\{n, 3\} + m
    $
    \\
    No.\ Experts (No.\ Leaves)
    &
    $
       \mathcal{O}\Big(
        v^{
            \big\lceil
                c\,\log(C)\big( 
                            1 
                        + 
                            \log(
                                \varepsilon^{2r/n}\,\operatorname{diam}( {\mathcal{K}} ))
                                /\big(
                                2\,\omega^{-1}\big(
                                    \frac{\varepsilon}{
                                    131 \, (nm)^{1/2}
                                    }
                                \big)  
                            \big)
                    \big)
            \big\rceil
        ,
        }
    \Big)
    $
    \\
    Height (Routing Complexity)
    &
    $
    \big\lceil
                c\,\log(C)\big( 
                            1 
                        + 
                            \log(\epsilon^{2r/n}\operatorname{diam}( {\mathcal{K}} )/\big(
                                2
                                \omega^{-1}\big(
                                    \frac{\varepsilon}{
                                    131 \, (nm)^{1/2}
                                    }
                                \big)  
                                \big))
                    \big)
            \big\rceil
    $
    \\
    Nodes & 
    $
        \mathcal{O}\Biggl(
        \frac{v^{
            \lceil
                c\,\log(C)( 
                            1 
                        + 
                            \log(\varepsilon^{2r/n}\operatorname{diam}( {\mathcal{K}} )
                                /
                                \big(2\omega^{-1}(
                                \frac{\varepsilon}{
                                131 \, (nm)^{1/2}
                                }
                            )  
                            \big)
                            )
                    )
            \rceil + 1
        } - 1}{
        \big\lceil
                c\log(C)\big( 
                            1 
                        + 
                            \log(
                            \frac{\varepsilon^{2r/n}}{2}
                                \operatorname{diam}( {\mathcal{K}} ))
                            /
                            \omega^{-1}\biggl(
                                \frac{\varepsilon}{
                                131 \, (nm)^{1/2}
                                }
                            \biggr)  
                    \big)
            \big\rceil - 1
        }
    \Biggr)
    $
    \\
    \arrayrulecolor{black}
    \bottomrule
    \end{tabular}
\end{table}

\section{Appendix: Proofs}
\label{a:Proof}
\subsection{{Proof of Theorem~\ref{thrm:Main}}}
\begin{proof}[{Proof of Lemma~\ref{lem:complete_n_ary_tree_counting}}]
Since $ {\mathcal{K}} $ is a doubling metric space then, \citep[Lemma 7.1]{acciaio2023designing}, for each $\delta>0$, there exist $x_1,\dots,x_N\in \mathcal{K}$ satisfying 
\[
    \max_{x\in \mathcal{K}}\,
        \min_{i=1,\dots,N_{\delta}}\, 
            \|x-x_i\|
    <
        \delta
    \mbox{ and }
        N_{\delta}
        \le 
        C^{\lceil 
            \log(\operatorname{diam}(\mathcal{K})/\delta)
        \rceil}
.
\]
In particular, since the doubling number of $\mathcal{K}$ is $C$, we have the upper-bound
\begin{equation}
\label{eq:bound_N__Balls}
        N_{\delta}
    \le 
        C\,C^{\log(\delta^{-1}\operatorname{diam}(\mathcal{K}))}
.
\end{equation}
An elementary computation shows that the complete $v$-ary tree of height $h$ has leaves $L$ and total vertices/nodes $V$ given by
\begin{equation}
\label{eq:counting_completevaryTree__nodes_n_leaves}
        L = v^h
    \mbox{ and }
       V = \frac{v^{h+1}-1}{h-1}
.
\end{equation}
Taking the formulation of $L$ given in Eq.~\eqref{eq:counting_completevaryTree__nodes_n_leaves}, to be the least \textit{integer} upper bound of the right-hand side of Eq.~\eqref{eq:bound_N__Balls}, which is itself an upper-bound for $N_{\delta}$, and solving for $h$ yields:
\allowdisplaybreaks
\begin{align}
\label{eq:identifying_height__UB}
        h 
    = &  
        \Big\lceil
            \log_v(C)\big( 
                        1 
                    + 
                        \log(\delta^{-1}\operatorname{diam}( {\mathcal{K}} ))
                \big)
        \Big\rceil
    ,
\end{align}
where the integer ceiling was applied since $h$ must be an integer.  

Let $\mathcal{L}$ be any set $v^h$ points in $ {\mathcal{K}} $ containing the set $\{x_i\}_{i=1}^{N_{\delta}}$.  Let $\mathcal{T}\eqdef(V,E)$ be any complete binary tree with leaves $\mathcal{L}$; note that, $\mathcal{L}\subseteq V$.  By construction, and the computation in Eq.~\eqref{eq:counting_completevaryTree__nodes_n_leaves}, $\mathcal{T}$ has $v^h$ leaves and $\frac{Lv-1}{h-1}$ nodes. 
\end{proof}

For completeness, we include a minor modification of the proof of Proposition \citep[Proposition 3.10]{acciaio2023designing}, which allows for the approximation of uniformly continuous functions of arbitrarily low regularity.  The original formulation of that result only allows for $\alpha$-H\"{o}lder function.  

\begin{proof}[{Proof of Lemma~\ref{lem_universal_approximation_improved_rates}}]
If $f(x)=c$ for some 
constant $c>0$, 
then the statement holds with the neural network $\hatf(x)=c$, which can be represented as in \eqref{eq_definition_ffNNrepresentation_function}
with $[d]=(n,m)$, where $A^j$ is the $0$ matrix for all $j$, and the ``$c$'' in \eqref{eq_definition_ffNNrepresentation_function} is taken to be this constant $c$.
Therefore, we henceforth only need to consider the case where $f$ is not constant.
Let us observe that, if we pick some $x^{\star}\in \mathcal{K}$, then for any multi-index $[d]$ and any neural network $\hatf_{\theta}\in \NN[[d]]$, $\hatf_{\theta}(x)-f(x^{\star})\in \NN[[d]]$, since $\NN[[d]]$ is invariant to post-composition by affine functions.  
Thus, we represent $\hatf_{\theta}(x)-f(x^{\star})=\hatf_{\theta^{\star}}(x)$, for some $\theta^{\star}\in \rr^{P([d])}$.  Consequently:
\[
    \sup_{x\in  {\mathcal{K}} }
    \left|
    \|
    (f(x)-f(x^{\star})) - \hatf_{\theta^{\star}}(x)
    \|
    -
    \|
    f(x) - \hatf_{\theta}(x)
    \|
    \right|
    =0
    .
\]
Therefore, without loss of generality, we assume that $f(x^*)=0$ for some $x^*\in  {\mathcal{K}} $.  
By \citet[Theorem 1.12]{BenyaminiLindenstrauss_2000_GNFABook}, there exists an $\omega$-uniformly continuous map $F:\rr^n\to\rr^m$ extending $f$.

\hfill\\
\noindent\textbf{Step 1 -- Normalizing $\tilde{f}$ to the Unit Cube:}
    First, we identify a hypercube ``nestling'' $ {\mathcal{K}} $.  To this end, let
    \begin{equation}
        \label{PROOF_prop_universal_tree__good_cube_via_Jungs_Radius}
        r_ {\mathcal{K}}  \eqdef \operatorname{diam}( {\mathcal{K}} ) \sqrt{ \frac{n}{2(n+1)} }.
    \end{equation}
    By Jung's Theorem (see \cite{Jung1901}), there exists $x_0\in \rr^n$ such that the closed Euclidean ball $\overline{\operatorname{Ball}_{(\rr^n,d_n)}\left(x_0,r_ {\mathcal{K}} \right)}$ contains $ {\mathcal{K}} $.
    Therefore, by H\"{o}lder's inequality, we have that the $n$-dimensional hypercube $[x_0-r_ {\mathcal{K}} \bar{1},x_0+r_ {\mathcal{K}} \bar{1}]$~
    \footnote{For $x,y \in \rr^n$ we denote by $[x,y]$ the hypercube defined by $\prod_{i = 1}^n [x_i, y_i]$.}
    contains $\overline{ B_{({\rr^n},d_n)}\left( x_0,r_ {\mathcal{K}}  \right) }$, where $\bar{1}=(1,\dots,1)\in \rr^n$.
    Consequently, $ {\mathcal{K}} \subseteq [x_0-r_ {\mathcal{K}} \bar{1},x_0+r_ {\mathcal{K}} \bar{1}]$.  
    Let $\tilde{f}\eqdef F|_{[x_0-r_ {\mathcal{K}} \bar{1},x_0+r_ {\mathcal{K}} \bar{1}]}$, 
    then $\tilde f \in C([x_0 - r_ {\mathcal{K}}  \bar 1, x_0 + r_ {\mathcal{K}}  \bar 1], \rr^m)$
    is an $\omega$-uniformly continuous extension of $f$ to $[x_0 - r_ {\mathcal{K}}  \bar 1, x_0 + r_ {\mathcal{K}}  \bar 1]$.

Since $ {\mathcal{K}} $ has at least two distinct points, then $r_ {\mathcal{K}} >0$.  Hence, the affine function
\[
T:\rr^n \ni x \mapsto 
(2r_ {\mathcal{K}} )^{-1}(x-x_0+r_ {\mathcal{K}} \bar{1}) \in \rr^n
\]
is well-defined, invertible,  not identically $0$, and maps $[x_0-r_ {\mathcal{K}} \bar{1},x_0-r_ {\mathcal{K}} \bar{1}]$ to $[0,1]^n$.  A direct computation shows that $g \eqdef \tilde f \circ T^{-1}$ is also uniformly continuous, whose modulus of continuity $\tilde{\omega}:[0,\infty)\to [0,\infty)$ is given by 
\begin{equation}
\label{eq:modulus_modified_function}
        \tilde{\omega}(t) 
    \eqdef 
        \omega(2r_ {\mathcal{K}} \,t)
\end{equation}
for all $t\in [0,\infty)$.  Furthermore, since for each $i=1,\dots,m$, define $\operatorname{pj}_i:\rr^m\ni y\mapsto y_i\in \rr$.  Since each $\operatorname{pj}_i$ is $1$-Lipschitz then, for each $i=1,\dots,m$, the map $g_i\eqdef \operatorname{pj}_i\circ g:[0,1]^n\to \rr$ is also $\tilde{\omega}$-uniformly continuous.  By orthogonality, we also note that $g(x) = \sum_{i=1}^m\, g_i(x)\,e_i$, for each $x\in \rr^n$, where $e_1,\dots,e_m$ is the standard orthonormal basis of $\mathbb{R}^m$; i.e.\ the $i^{th}$ coordinate of $e_j$ is $1$ if and only if $i=j$ and $0$ is otherwise.

\hfill\\
\noindent\textbf{Step 2 -- Constructing the Approximator:}
For $i=1,\dots,m$, let $\hat f_{\theta^{(i)}}\in \NN[[d^{(i)}]]$ for some multi-index $[d^{(i)}]=(d_0^{(i)},\dots,d_J^{(i)})$ with $n$-dimensional input layer and $1$-dimensional output layer, i.e.\ $d_0^{(i)}=n$ and $d_J^{(i)}=1$, and let $\theta^{(i)} \in \rr^{P([d^{(i)}])}$ be the parameters defining $\hat f_{\theta^{(i)}}$.
Since the pre-composition by affine functions and the post-composition by linear functions of neural networks in $\NN[[d^{(i)}]]$ are again neural networks in $\NN[[d^{(i)}]]$, we have that $g_{\theta^{(i)}} \eqdef \hat f_{\theta^{(i)}} \circ T^{-1}$ belongs to  $\NN[[d^{(i)}]]$.
Denote the standard basis of $\rr^m$ by $\{e_i\}_{i=1}^m$.
We compute:
\allowdisplaybreaks
\begin{align}
    \nonumber
    & \sup_{x\in  {\mathcal{K}} }\, 
    \left\|f(x)-\sum_{i=1}^m \hatf_{\theta^{(i)}}(x)e_i\right\|
\\  
\nonumber
    = &
    \sup_{x\in  {\mathcal{K}} }\, 
    \left\|\tilde{f}(x)
        -
    \sum_{i=1}^m \hatf_{\theta^{(i)}}(x)e_i\right\|
\\
\nonumber
    \le & 
    \sup_{x\in [x_0-r_ {\mathcal{K}} \bar{1},x_0+r_ {\mathcal{K}} \bar{1}]}\, 
    \left\|\tilde{f}(x)
        -
    \sum_{i=1}^m \hatf_{\theta^{(i)}}(x)e_i\right\|
\\
    \nonumber
    = &
    \sup_{x\in [x_0-r_ {\mathcal{K}} \bar{1},x_0+r_ {\mathcal{K}} \bar{1}]}\, 
    \left\|\tilde{f}\circ T^{-1}\circ T(x)
        -
    \sum_{i=1}^m \hatf_{\theta^{(i)}}\circ T^{-1}\circ T(x)e_i\right\|
\\
    \nonumber
        = &
    \sup_{u\in [0,1]^n}\, 
    \left\|\sum_{i=1}^m g_i(u)\,e_i - \sum_{i=1}^m g_{\theta^{(i)}}(u) e_i\right\|
\\
    \le & 
            \sqrt{m}
    \max_{u\in [0,1]^n}\, 
    \max_{1\leq i\leq m}
    \left|
        g_i(u)
            -
        g_{\theta^{(i)}} (u)
    \right|
    \label{PROOF_prop_universal_tree__main_estimate_Part_A}
    .
\end{align}
Fix $\tilde{\varepsilon}>0$, to be determined below.  
For each $i=1,\dots,m$, depending on which assumption $\sigma$ satisfies, \citep[Theorem 1.1]{shen2022optimal} (resp.~\citep[Theorem 1]{shen2022deep_JMLR} if $\sigma$ is as in Definition~\ref{eq:activation_superexpressive__nontrainable})  imply that there is a neural network with activation function $\sigma^{\star}:\rr\to \rr$ satisfying
\begin{equation}
\label{eq:estimate}
    \max_{u\in [0,1]^n}\, 
        \left|
            g_i(u)
                -
            g_{\theta^{(i)}} (u)
        \right|
    <
        \tilde{\varepsilon}
    .
\end{equation}
Furthermore, the depth and width of these MLPs can be bounded above on a case-by-case basis as follows:
\begin{enumerate}
    \item[(i)] If $\sigma$ satisfies Definition~\ref{defn_TrainableActivation_SuperExpressive} then, setting each $\alpha=0$ implies that $\sigma_{0}(x)=\sigma^{\star}(x)$, as defined in Eq.~\eqref{eq:activation_superexpressive__nontrainable}; thus
    \begin{equation}
    \nonumber
    \begin{aligned}
    J^{(i)}
    \le 
        11
    \,\mbox{ and }\,
        \max_{1 \le j \le J^{(i)}} d_j 
    \le 
        36 n (2n +1)
\end{aligned}
\end{equation}
In this case, we set $\tilde{\varepsilon}\eqdef \varepsilon/\sqrt{m}$; we have used~\citet[Theorem 1]{shen2022deep_JMLR}.
    \item[(ii)] If $\sigma$ satisfies Definition~\ref{defn:PReLU} then, setting each $\alpha=1$ implies that $\sigma_{0}(x)=\operatorname{ReLU}(x)\eqdef \max\{0,x\}$; yielding
    \begin{equation}
    \nonumber
    \begin{aligned}
    J^{(i)}
    \le 
        18 + 2n  + 
        11
        \Big\lceil
        \Big(
            \frac{2r_ {\mathcal{K}} }{
            \omega^{-1}\big(
                \varepsilon/(131 \sqrt{n\,m})
            \big)
            }
        \Big)^{n/2}
        \Big\rceil
    \,\mbox{ and }\,
        \max_{1 \le j \le J^{(i)}} d_j 
    \le 
        16 \, 
        \max\{n, 3\}
\end{aligned}
\end{equation}
in this case, we have employed~\citet[Theorem 1.1]{shen2022optimal}.
\end{enumerate}
In either case, the estimate in Eq.~\eqref{PROOF_prop_universal_tree__main_estimate_Part_A} yields
\[
    \max_{x\in  {\mathcal{K}} }\, 
        \left\|f(x)-\sum_{i=1}^m \hatf_{\theta^{(i)}}(x)e_i\right\|
    <
        \varepsilon
    .
\]

\noindent\paragraph{Step 3 -- Assembling into an MLP:}
Let $g_1\bullet g_2$ 
denote the component-wise composition of a univariate function $g_1$ with a multivariate function $g_2$.
 
If the activation function $\sigma$ is either in Definitions~\ref{defn_TrainableActivation_SuperExpressive} or \ref{defn:PReLU}, then it trivially implements the identity $I_{\mathbb{R}}$ on $\mathbb{R}$ by setting $\alpha=0$; i.e.\ $\sigma_{1}=I_{\mathbb{R}}$.  Consequentially, for any $k\in \N_+$, if $I_k$ denotes the $k\times k$-identity matrix, then
$I_k \sigma_1\bullet I_k \in \NN[[d_k]]$ with $P([d])=2k$, and $I_k \sigma_1\bullet I_k=1_{\rr^k}$.
Therefore, mutatis mutandis, $\NN[[\cdot]]$ satisfies the \textit{$c$-identity requirement with%
\footnote{Formally, it satisfies what is the $1$-identity requirement, thus it satisfies the $c$-identity requirement for all integers $c\ge 2$.  However, the authors of \citet{FlorianHighDimensional2021} do not explicitly consider the extremal case where $c=1$.}~%
$c=2$, as defined in} \citet[Definition 4]{FlorianHighDimensional2021}.  
From there, mutatis mutandis, we may apply \citet[Proposition 5]{FlorianHighDimensional2021}.  Thus, there is a multi-index $[d]=(d_0,\dots,d_J)$ with $d_0=n$ and $d_J=m$, and a network $\hatf_{\theta}\in \NN[[d]]$ implementing $\sum_{i=1}^m \hatf_{\theta^{(i)}}e_i$, i.e.
\[
    \sum_{i=1}^m \hatf_{\theta^{(i)}}e_i
        = 
    \hatf_{\theta},
\]
such that $\hatf_{\theta}$'s depth and width are bounded-above, on a case-by-case basis, by
\begin{enumerate}
    \item[(i)] If $\sigma$ satisfies Definition~\ref{defn_TrainableActivation_SuperExpressive} then, setting each $\alpha=0$
    \begin{equation}
    \nonumber
    \begin{aligned}
    J
    \le &
        15\,m
    \\
    \,\mbox{ and }\,
        \max_{1 \le j \le J^{(i)}} d_j 
    \le &
        36 n (2n +1 ) + m
.
\end{aligned}
\end{equation}
In this case, we set $\tilde{\varepsilon}\eqdef \varepsilon/\sqrt{m}$.
    \item[(ii)] If $\sigma$ satisfies Definition~\ref{defn:PReLU} then, setting each $\alpha=0$ yields
    \begin{equation}
    \nonumber
    \begin{aligned}
    J
    \le &
        m\left(
            19 + 2n  + 
            11
            \Big\lceil
            \Big(
                \frac{2r_ {\mathcal{K}} }{
                \omega^{-1}\big(
                    \varepsilon/(131 \sqrt{n\,m})
                \big)
                }
            \Big)^{n/2}
            \Big\rceil
        \right)
    \\
    \,\mbox{ and }\,
        \max_{1 \le j \le J^{(i)}} d_j 
    \le &
        16 \, 
        \max\{n, 3\}
        +
        m
.
\end{aligned}
\end{equation}
\end{enumerate}
Incorporating the definition of $r_ {\mathcal{K}} $ in Eq.~\eqref{PROOF_prop_universal_tree__good_cube_via_Jungs_Radius} and employing the inequality $\operatorname{diam}( {\mathcal{K}} )\le 2 \operatorname{rad}( {\mathcal{K}} )$ completes the proof.
\end{proof}

\begin{lemma}[Trade-Off: No.\ Expert vs. Expert Complexity]
\label{lem:technicalversion_max}
Let $ {\mathcal{K}} $ be a compact subset of $\rr^n$ whose doubling number is $C$ and a uniformly continuous map $f: {\mathcal{K}} \to \rr^m$ with modulus of continuity $\omega$.

\noindent Fix $v\in \mathbb{N}$ with $v\ge 2$, $0<\delta \le \operatorname{rad}( {\mathcal{K}} )$, and $\varepsilon>0$.  Suppose that $\sigma$ satisfies Definition~\ref{defn:PReLU}.
There exists a $p\in \mathbb{N}_+$ and a $v$-ary tree $\mathcal{T}\eqdef(V,E)$ with leaves $\mathcal{L}\eqdef \{(v_i,\theta_i)\}_{i=1}^L \subseteq  {\mathcal{K}} \times\rr^p$ satisfying
\begin{equation}
\label{eq:packing_condition_II}
    \max_{x\in  {\mathcal{K}} }\,
        \min_{(v_i,\theta_i)\in \mathcal{L}}\,
            \|
                    x
                -
                    v_i
            \|
        <
            \delta
.
\end{equation}
Furthermore, for each $x\in  {\mathcal{K}} $ and each $i=1,\dots,L$, if $\|x-v_i\|<\delta$ then
\[
    \|
            f(x)
        -
            f_{\theta_i}(x)
    \|
    <
        \varepsilon
    .
\]
We have the following estimates:
\begin{enumerate}
    \item[(i)] \textbf{Depth.} Depth of each $\hatf_{\theta_i}$ is 
    $
        m\left(
            19 + 2n  + 
            11
            \Big\lceil
            \Big(
                \frac{
                \delta 2^{3/2}
                n^{1/2}
                }{
                (n+1)^{1/2}
                \omega^{-1}\big(
                    \varepsilon/(131 \sqrt{n\,m})
                \big)
                }
            \Big)^{n/2}
            \Big\rceil
        \right)
    $
\item[(ii)] \textbf{Width.} Width of each $\hatf_{\theta_i}$ is $16 \, \max\{n, 3\} + m$
\item[(iii)] \textbf{Leaves}: at most $
        L 
    =
        v^{
            \big\lceil
                c\,\log(C)\big( 
                            1 
                        + 
                            \log(\delta^{-1}\operatorname{diam}( {\mathcal{K}} ))
                    \big)
            \big\rceil
        ,
        }$
    \item[(iv)] \textbf{Height}: $\big\lceil
                c\,\log(C)\big( 
                            1 
                        + 
                            \log(\delta^{-1}\operatorname{diam}( {\mathcal{K}} ))
                    \big)
            \big\rceil$,
    \item[(v)] \textbf{Nodes}: At most $
        \frac{v^{
            \lceil
                c\,\log(C)( 
                            1 
                        + 
                            \log(\delta^{-1}\operatorname{diam}( {\mathcal{K}} ))
                    )
            \rceil + 1
        } - 1}{
        \big\lceil
                c\log(C)\big( 
                            1 
                        + 
                            \log(\delta^{-1}\operatorname{diam}( {\mathcal{K}} ))
                    \big)
            \big\rceil - 1
        }
    $
\end{enumerate}
where $c\eqdef \log(v)^{-1}$.  
\end{lemma}

\begin{proof}[{Proof of Lemma~\ref{lem:technicalversion_max}}]
Consider the tree $\tilde{\mathcal{T}}$ given by Lemma~\ref{lem:complete_n_ary_tree_counting}.  For each leaf $v_i$ of $\tilde{\mathcal{T}}$, we apply Lemma~\ref{lem_universal_approximation_improved_rates}
to deduce the existence of an MLP $\hatf_{\theta_i}$ with explicit depth and width estimates given by that lemma, satisfying the uniform estimate
\begin{equation}
\label{eq:unif_estimate}
    \max_{\|x-v_i\|\le \delta}\,
        \|f(x)-\hatf_{\theta_i}(x)\|
    <
        \varepsilon
    .
\end{equation}
Let $\mathcal{T}$ be the same tree as $\tilde{\mathcal{T}}$ with leaves identified with $\{(v_i,\theta_i)\}_{i=1}^L$.  
\end{proof}
We now prove our main technical version, namely Theorem~\ref{thm:Main__TechnicalVersion}, which directly implies the special case recorded in Theorem~\ref{thrm:Main}.  

\begin{proof}[{Proof of Theorem~\ref{thm:Main__TechnicalVersion} (and this Theorem~\ref{thrm:Main})}]
Applying Lemma~\ref{lem:technicalversion_max} with $\delta$ given as the solution of 
\begin{equation}
\label{eq:implicit_deltavalue}
\Big(\frac{\delta 2^{3/2}n^{1/2}}{(n+1)^{1/2}\omega^{-1}(\varepsilon/131\,(nm)^{1/2})} \Big)^{n/2}  \le \Big(\frac{\delta 2^{3/2}n^{1/2}}{(2n^{1/2}\omega^{-1}(\varepsilon/131\,(nm)^{1/2})} \Big)^{n/2} = \varepsilon^{-r}
.
\end{equation}
Solving Eq.~\eqref{eq:implicit_deltavalue} for $\delta$ implies that it is given by
\[
        \delta 
    =   
        \frac{
            \varepsilon^{-2r/n}
        }{2}
        \,
        \omega^{-1}\biggl(
            \frac{\varepsilon}{
            131 \, (nm)^{1/2}
            }
        \biggr)  
.
\]
This completes the proof.
\end{proof}

\begin{proof}[{Proof of Theorem~\ref{thrm:Main}}]
Setting ${\mathcal{K}} = B_n(0,1)$, $r=1/2$, $\omega(t)=Lt$, and thus $\omega^{-1}(t)=L^{-1}\,t^{1/\alpha}$, in Theorem~\ref{thm:Main__TechnicalVersion} yields the conclusion.  
Finally, by the computation in Example~\ref{ex:doubling_Ball__Euclidean}, we have that $C\le 2^{n+1}$; thus, $\log(C)=(n+1)\log(2)
\le 2n$.  Noting that $c=1/\log(2)=1$ completes the proof.
\end{proof}


\subsection{{Proofs of  Theorem~\ref{thrm:Bounded_VC__NeuralPathways}}}
\label{a:Proofs_VC_DimensionResults}

We use the following lemma and its proof due to \textit{Gabriel Conant}.  
\begin{lemma}[{\cite{UpperBoundVCPartitioning_MO}}]
\label{lem:VCParitioning}
Fix $n,L,d\in \mathbb{N}_+$.  
Let $\mathcal{H}$ be a non-empty set of functions from $\mathbb{R}^n$ to $\{0,1\}$ of VC dimension at-most $d$.  
Let $\mathcal{C}_L$ be the set of all ordered partitions (Voronoi diagrams) $(C_l)_{l=1}^{\tilde{L}}$ covering $\mathbb{R}^n$, where $\tilde{L}\le L$, and for which there exist \textit{distinct} $p_1,\dots,p_{\tilde{L}}\in \mathbb{R}^n$ such that: for each $l=1,\dots,\tilde{L}$
\begin{equation}
\label{eq:paritioning_Lemma}
\begin{aligned}
        C_l & 
    \eqdef 
        \tilde{C}_l \setminus \bigcup_{s<l}\, \tilde{C_s}
\\ 
        \tilde{C}_l & 
    \eqdef 
        \{x\in \mathbb{R}^n:\, \|x-p_l\|=\min_{s=1,\dots,\tilde{L}}\, \|x-p_s\|\} 
.
\end{aligned}
\end{equation}
Let $\mathcal{H}_{L}$ be the set of functions from $\mathbb{R}^n$ to $\{0,1\}$ of the form
\[
        f
    =
        \sum_{l=1}^{\tilde{L}}
        \,
        f_l\,I_{C_l}
\]
where $\tilde{L}\in \mathbb{N}_+$ with $\tilde{L}\le L$, $f_1,\dots,f_{\tilde{L}}\in \mathcal{H}$, and $(C_l)_{l=1}^{\tilde{L}} \in \mathcal{C}_L$.  Then, the VC dimension of $\mathcal{H}_L$ satisfies
\[
        \operatorname{VC}(
            \mathcal{H}_L
        )
    \le 
        8
        L
        \log(\max\{2,L\})^2
        \,
        \big(
            \max\{d
                ,
                2(n+1)\, (L-1)\, \log(3L - 3)
            \}
        \big) 
.
\]
\end{lemma}
\begin{proof}[{Proof of Lemma~\ref{lem:VCParitioning}}]
Let us first fix our notation.  
Each $x_0\in \mathbb{R}^n\setminus\{0\}$ and $t\in \mathbb{R}$ defines a \textit{halfspace} in $\mathbb{R}^n$ given by $HS_{x_0,t}\eqdef \{x\in \mathbb{R}^n:\, \langle x_0,x\rangle \le t\}$ (see \citep[Section 2.2.1]{BoydVandenberghe_ConvexOptim_Book_2004} for details).  We denote set of all halfspaces in $\mathbb{R}^n$ by $\mathcal{HS}_n\eqdef \{HS_{x_0,t}:\,\exists x_0\in \mathbb{R}^n\setminus\{0\}\,\exists t\in \mathbb{R}\}$.
Consider the set $\mathcal{C}(L)$ of all $C \subseteq \mathbb{R}^n$ of the form
\begin{equation}
\label{eq:intersection_of_halfspaces}
        C
    =
        \bigcap_{l=1}^{\tilde{L}-1}
        \,
        X_l
\end{equation}
for some positive integer $2\le \tilde{L}\le \max\{2,L\}$ and $X_1,\dots,X_{\tilde{L}-1}\in \mathcal{HS}_n$.
\paragraph{Step 1 - Reformulation as Set of Sets}
\hfill\\
By definition of the powerset $2^{\mathbb{R}^n}$ of the set $\mathbb{R}^n$, each subset $A\subseteq \mathbb{R}^n$ can be identifies with a function (classifier) from $\mathbb{R}^n$ to $\{0,1\}$ via the bijection mapping any $X\in 2^{\mathbb{R}^n}$ to the binary classifier $I_X$ (i.e.\ the indicator function of the set $X$).
Using this bijection, we henceforth identify both $\mathcal{H}$ and $\mathcal{H}_L$ with subsets of the powerset $2^{\mathbb{R}^n}$.

Under this identification, the class $\mathcal{H}_L$ can be represented as the collection of subsets $X$ of $\mathbb{R}^n$ of the form
\begin{equation}
\label{eq:SetOfSetsDescription}
        X
    =
        \bigcup_{l=1}^{\tilde{L}}
        \,
        H_l\cap C_l
,
\end{equation}
where $\tilde{L}\in \mathbb{N}_+$ satisfies $\tilde{L}\le L$, and for each $l=1,\dots,\tilde{L}$ we have $H_l\in \mathcal{H}$ and $(C_l)_{l=1}^{\tilde{L}}\in \mathcal{C}_L$ is of the form~\eqref{eq:paritioning_Lemma} for some \textit{distinct} points $p_1,\dots,p_{\tilde{L}}\in \mathbb{R}^n$.
\paragraph{Step 2 - VC Dimension of Voronoi Diagrams with at-most $L$ Cells}
\hfill\\
An element of $(C_l)_{l=1}^{\tilde{L}}$ of $\mathcal{C}_L$ is, by definition, a Voronoi diagram in $\mathbb{R}^n$ and thus, \citet[Exercise 2.9]{BoydVandenberghe_ConvexOptim_Book_2004} implies that each $C_1,\dots,C_{\tilde{L}}$  is the intersection of $\tilde{L}-1\le L-1$ halfspaces; i.e.~$C_1,\dots,C_{\tilde{L}}\in \mathcal{C}(L)$ (see Eq.~\eqref{eq:intersection_of_halfspaces}).  
Since $\mathcal{C}_L = \{\cap_{l=1}^{\tilde{L}}\, H_i:\, \exists \tilde{L}\in \mathbb{N}_+\, H_1,\dots,H_{\tilde{L}}\in \mathcal{HS}_n\, \tilde{L}\le L\}$ then \citet[Lemma 3.2.3]{blumer1989learnability} implies that
\begin{equation}
\label{eq:VCDimension_VoronoiCells_atmostL_Parts}
\begin{aligned}
        \operatorname{VC}(\mathcal{C}_L)
    \le &
        2\operatorname{VC}(\mathcal{HS}_n)\, (L-1)\, \log(3L - 3)
\\
    \le & 
        2(n+1)\, (L-1)\, \log(3L - 3)
;
\end{aligned}
\end{equation}
the second inequality in Eq.~\eqref{eq:VCDimension_VoronoiCells_atmostL_Parts} holds since $\operatorname{VC}(\mathcal{HS}_n) = n+1$ by~\citet[Theorem 9.3]{shalev2014understanding}.  
\paragraph{Step 3 - VC Dimension of The Class $\mathcal{H}_L$}
\hfill\\
Define $\mathcal{H}\cap \mathcal{C}(L)\eqdef \{H\cap C:\, H\in \mathcal{H}\mbox{ and } C\in \mathcal{C}(L)\}$.  Again using \citet[Lemma 3.2.3]{blumer1989learnability}, we have
\begin{equation}
\label{eq:intesection_bound}
\begin{aligned}
        \operatorname{VC}\big(\mathcal{H}\cap \mathcal{C}(L)\big)
    \le &
        2
        \, 
        \big(
            \max\{\operatorname{VC}(\mathcal{H}),\operatorname{VC}(\mathcal{C}(L))\}
        \big)
        \,
        2
        \,
        \log(6)
\\
    \le &
        4
        \log(6)
        \, 
        \big(
            \max\{\operatorname{VC}(\mathcal{H}),\operatorname{VC}(\mathcal{C}(L))\}
        \big)
\\
    \le &
        4
        \log(6)
        \, 
        \big(
            \max\{d
                ,
                2(n+1)\, (L-1)\, \log(3L - 3)
            \}
        \big)   
.
\end{aligned}
\end{equation}
Consider the set $\tilde{\mathcal{H}}\eqdef \{\cup_{l=1}^{\tilde{L}}\,H_l:\, \tilde{L}\in \mathbb{N}_+,\,\tilde{L}\le L,\,  \forall l=1,\dots, \tilde{L} ,\, H_l\in \mathcal{H}\cap \mathcal{C}(L)\}$.  Applying \citet[Lemma 3.2.3]{blumer1989learnability}, one final time yields
\begin{align}
\label{eq:union_bound__BEGIN}
        \operatorname{VC}\big(\tilde{\mathcal{H}}\big)
    \le &
        2
        \operatorname{VC}\big(\mathcal{H}\cap \mathcal{C}(L)\big)
        \, 
        L
        \log(3L)
\\
\label{eq:union_bound}
    \le &
        2
        \Big(
            4
            \log(6)
            \, 
            \big(
                \max\{d
                    ,
                    2(n+1)\, (L-1)\, \log(3L - 3)
                \}
            \big) 
        \Big)
        \, 
        L
        \log(3L)
\\
\label{eq:union_bound__END}
    \le &
        8
        L
        \log(\max\{2,L\})^2
        \,
        \big(
            \max\{d
                ,
                2(n+1)\, (L-1)\, \log(3L - 3)
            \}
        \big) 
\end{align}
where Eq.~\eqref{eq:union_bound} held by the estimate in Eq.~\eqref{eq:intesection_bound}.  Finally, since $\operatorname{VC}(A)\le \operatorname{VC}(B)$ whenever $A\subseteq B$ for any set $B$ then since $\mathcal{H}_L\subseteq \tilde{\mathcal{H}}$ then Eq.~\eqref{eq:union_bound__BEGIN}-\eqref{eq:union_bound__END} yields the desired conclusion.
\end{proof}

We may now derive Theorem~\ref{thrm:Bounded_VC__NeuralPathways} by merging Lemma~\ref{lem:VCParitioning} and one of the main results of \citet{bartlett2019nearly}.
\begin{proof}[{Proof of Theorem~\ref{thrm:Bounded_VC__NeuralPathways}}]
Let $n,J,W,L\in \mathbb{N}_+$ and consider the (non-empty) set of real-valued functions $\mathcal{NN}_{J,W:n,1}^{\operatorname{PReLU}}$.
By definition of the VC dimension of a set of real-valued functions, given circa Eq.~\eqref{eq:real_valued_shattering}, we have 
\begin{align}
\nonumber
            \operatorname{VC}\big(\mathcal{NN}_{J,W:n,1}^{\operatorname{PReLU}}\big) 
        \eqdef  &
            \operatorname{VC}\big(I_{(0,\infty)}\circ \mathcal{NN}_{J,W:n,1}^{\operatorname{PReLU}}\big) 
\\
\label{eq:VC_NeuralPathwaysReminder}
            \operatorname{VC}\big(\mathcal{NP}_{J,W,L:n,1}^{\operatorname{PReLU}}\big) 
        \eqdef  &
            \operatorname{VC}\big(I_{(0,\infty)}\circ \mathcal{NP}_{J,W,L:n,1}^{\operatorname{PReLU}}\big) 
.
\end{align}
By \citet[Theorem 7]{bartlett2019nearly}, we have that
\begin{equation}
\label{eq:VCDim_NNs}
        \operatorname{VC}(I_{(0,\infty)}\circ \mathcal{NN}_{J,W:n,1}^{\operatorname{PReLU}})
    \le 
        D^{\star}
    \eqdef 
        \big\lceil
            J+(J+1)\,W^2 
            \, 
            \log_2\big(
                e\,4 (J+1) \, W
                    \log_2(
                        e2(J+1)W
                    )
            \big)
        \big\rceil
.
\end{equation}
Therefore, applying Lemma~\ref{lem:VCParitioning} with $\mathcal{H}=  \big(I_{(0,\infty)}\circ \mathcal{NN}_{J,W:n,1}^{\operatorname{PReLU}}\big)$
yields the estimate 
\begin{equation}
\label{eq:Lemma_plus_Bartlett}
        \operatorname{VC}\big(I_{(0,\infty)}\circ \mathcal{NP}_{J,W,L:n,1}^{\operatorname{PReLU}}\big)
    \le 
        8
        L
        \log(\max\{2,L\})^2
        \,
        \big(
            \max\{D^{\star}
                ,
                2(n+1)\, (L-1)\, \log(3L - 3)
            \}
        \big) 
\end{equation}
Combining Eq.~\eqref{eq:Lemma_plus_Bartlett} and the definition~\eqref{eq:VC_NeuralPathwaysReminder} yields the bound.  In particular,
\[
        \operatorname{VC}\big( \mathcal{NP}_{J,W,L:n,1}^{\operatorname{PReLU}}\big)
    \in 
        \mathcal{O}\big(
            L\log(L)^2
            \,
            \max\{
                    nL\log(L)
                ,
                    JW^2\log(JW)      
            \}
        \big)
\]
yielding the second conclusion.
\end{proof}

\subsection{{Proof of Proposition~\ref{prop:Unbounded_VC__SuperExpressive}}}

\begin{proof}
We argue by contradiction.  
Suppose that $\mathcal{F}$ has finite VC dimension $\operatorname{VC}(\mathcal{F})$.  
Then, \citet[Theorem 2.4]{shen2022optimal} implies that there exists a $1$-Lipschitz map $f:[0,1]^n\to \mathbb{R}$ such that does not exist a \textit{strictly positive} $\varepsilon\in (0,\operatorname{VC}(\mathcal{F})^{-1/n}/9)$ satisfying
such that
\begin{equation}
\label{eq:PRF_prop:Unbounded_VC__SuperExpressive___ContradictionSetup}
        \inf_{\hatf\in \mathcal{F}}
        \,
        \sup_{x\in [0,1]^n}\, |\hatf(x)-f(x)| 
    \le 
        \varepsilon
.
\end{equation}
However, \citet[Theorem 1]{shen2022deep_JMLR} implies that, for every $1$-Lipschitz function, in particular for $f$, and for each $\tilde{\varepsilon}>0$ there exists a $\hatf_{\tilde{\varepsilon}}\in \mathcal{F}$ satisfying
\begin{equation}
\label{eq:PRF_prop:Unbounded_VC__SuperExpressive___Contradiction}
        \sup_{x\in [0,1]^n}\, |\hatf_{\tilde{\varepsilon}}(x)-f(x)| 
    \le 
        \tilde{\varepsilon}
.
\end{equation}
Setting $\tilde{\varepsilon}= \operatorname{VC}(\mathcal{F})^{-1/n}/18$ yields a contradiction as Eq.~\eqref{eq:PRF_prop:Unbounded_VC__SuperExpressive___ContradictionSetup} and Eq.~\eqref{eq:PRF_prop:Unbounded_VC__SuperExpressive___Contradiction} cannot both be simultaneously true.  Therefore, $\mathcal{F}$ has infinite VC dimension.
\end{proof}

\section{The Curse of Irregularity}
\label{a:COI}
We now explain why learning H\"{o}lder functions of low regularity ($(1,1/d)$-H\"{o}lder) functions on the real line segment $[0,1]$ is equally challenging as learning regular functions ($1$-Lipschitz) on $[0,1]^n$.

\subsection{H\"{o}lder Functions}
\label{a:COI__HolderExplained}
Fix $n,m\in \mathbb{N}$ and let $\mathcal{X}\subset \mathbb{R}^n$ be non-empty and compact of diameter $D$.  Fix $0<\alpha \le 1$, $L\ge 0$, and let $f:\mathcal{X}\rightarrow \mathbb{R}^m$ be $(\alpha,L)$-H\"{o}lder continuous, meaning
\[
        \|
            f(x)-f(y)
        \|
    \le
        \,
        L
        \|
            x - y 
        \|^{\alpha}
\]
holds for each $x,y\in \mathcal{X}$.  For any $L\ge 0$ and $0<\alpha\le 1$, we denote set of all $(\alpha,L)$-H\"{o}lder functions from $\mathcal{X}$ to $\mathbb{R}^n$ is denoted by $C^{\alpha}([0,1]^n,\mathbb{R};L)$. 

We focus on the class of locally H\"{o}lder functions since they are generic, i.e.\ universal, amongst all continuous functions by the Stone–Weierstrass theorem.  
In this case, H\"{o}lder functions are sufficiently rich to paint a full picture of the hardness to approximate arbitrary H\"{o}lder functions either by MLPs against the proposed model.  

In contrast to smaller generic function classes, such as polynomials, H\"{o}lder functions provide more freedom in experimentally visualizing our theoretical results.  
This degree of freedom is the parameter $\alpha$, which modulates their \textit{regularity}. As $\alpha$ tends to $0$ the H\"{o}lder functions become complicated and when $\alpha=1$ the Rademacher-Stephanov theorem, see \citep[Theorems 3.1.6 and 3.1.9]{Federer_GeometricMeasureTheory_1978} characterizes Lebesgue almost-everywhere differentiable functions as locally $(1,L)$-H\"{o}lder maps.  Note that $(1,L)$-H\"{o}lder functions are also called $L$-Lipschitz maps and, in this case, $L=\sup_{x}\,\|\nabla f(x)\|_{op}$ where the supremum is taken over all points where $f$ is differentiable and where $\|\cdot\|_{op}$ is the operator norm. 

\subsection{The Curse of Irregularity}
\label{a:COI__COIExplained}
The effect of \textit{low H\"{o}lder regularity}, i.e. when $\alpha\approx 0$, has the same effect as high-dimensionality on the approximability of arbitrary $\alpha$-H\"{o}lder functions.  
This is because any real-valued model/hypothesis class $\mathcal{F}_1$ of functions on $\mathbb{R}$ approximating an arbitrary $(\frac1{d},1)$-H\"{o}lder functions  By \citep[Theorem 2.4]{shen2022optimal}, we have the lower minimax bound: if for each $\varepsilon>0$ we have \textit{``the curse of irregularity''}
\begin{equation}
\label{eq:approximability__Regularity}
        \underset{f
        }{\sup}
        \,
        \underset{\hatf\in \mathcal{F}_1}{\inf}
        \,
        \underset{0\le x\le 1}{\sup}
        \,
                |f(x)-\hatf(x)|
        \le \varepsilon
    \Rightarrow
        \operatorname{VC}(\mathcal{F}_1) \in \Omega(\varepsilon^{-d})
\end{equation}
where the supremum is taken over all $f \in C^{1/d}([0,1],\mathbb{R};1)$.
The familiar curse of dimensionality also expresses the hardness to approximate an arbitrary $1$-Lipschitz ($(1,1)$-H\"{o}lder), thus relatively regular, function on $[0,1]^n$.  As above, consider any model/hypothesis class $\mathcal{F}_2$ of real-valued maps on $\mathbb{R}^d$ then, again using \citep[Theorem 2.4]{shen2022optimal}, one has the lower-bound
\begin{equation}
\label{eq:approximability__Dimensionality}
        \underset{f
        }{\sup}
        \,
        \underset{\hatf\in \mathcal{F}_2}{\inf}
        \,
        \underset{x\in [0,1]^n}{\sup}
        \,
                \|f(x)-\hatf(x)\|
        \le \varepsilon
    \Rightarrow
        \operatorname{VC}(\mathcal{F}_2) \in \Omega(\varepsilon^{-d})
\end{equation}
where the supremum is taken over all $f\in C^{1}([0,1]^n,\mathbb{R};1)$.  
Comparing Eq.~\eqref{eq:approximability__Regularity} and Eq.~\eqref{eq:approximability__Dimensionality}, we find that the difficulty of uniformly approximating an arbitrary \textit{low-regularity} ($(\frac1{d},1)$-H\"{o}lder) function on a $1$-dimensional domain is roughly just as complicated as approximating a relatively regular ($1$-Lipschitz) function on a \textit{high-dimensional} domain.  

Incorporating these lower bounds with the lower-bound in Eq.~\eqref{eq:VCDim_MLPs}, we infer that the minimum number layers $(L)$ and minimal width $(W)$ of each MLP approximating a low-regularity function on the low-dimensional domain $[0,1]$ is roughly the same as the minimal number of layers and width of an MLP approximating a high-regularity map on the high-dimensional domain $[0,1]^n$.

\section{{Does one Need Overparameterized Experts if there are Enough Experts?}}
\label{a:Experiments}

We evaluate our approach in two standard machine learning tasks: regression and classification. We experimentally show that MoMLPs, which distribute predictions over multiple neural networks, are competitive with a single large neural network containing as many model parameters as all the MoMLPs combined. This is desirable in cases where the large neural network does not fit into the memory of a single machine. On the contrary, the MoMLP model can be trained by distributing each MoMLP on separate machines (or equivalently serially on a single machine). Inference can then be performed by loading only a single MoMLP at a time into the GPU. Our source code to reproduce our results is available in the supplementary material.
\vspace{-0.15cm}
\subsection{Regression}

We first consider regression, where the goal is to approximate non-convex synthetic functions often used for performance test problems. In particular, we choose 1-dimensional H\"{o}lder functions, as well as Ackley \cite{ackley1987connectionist} and Rastrigin \cite{rastrigin1974systems} functions, whose formulations are detailed in Appendix~\ref{sec:experimental_details}.

\textbf{1D H\"{o}lder Functions.} We illustrate our primary finding, as encapsulated in Theorem~\ref{thrm:Main}, by leveraging $1$-dimensional functions characterized by very low regularity. This choice is motivated by the jagged structure inherent in such functions, necessitating an exponentially higher sampling frequency compared to smooth functions for achieving an accurate reconstruction. Indeed, this crucial sampling step forms the foundation of many quantitative universal approximation theorems \cite{yarotsky2017error,shen2021deep,kratsios2022universal}. As elaborated in Appendix~\ref{a:COI}, approximating a well-behaved (Lipschitz) function in $d$ dimensions poses a challenge equivalent to approximating a highly irregular function ($1/d$-H\"{o}lder) in a single dimension. A visual representation of a $1/d$-H\"{o}lder function is presented in Figure~\ref{fig:Holder_lowavshigha}, exemplified by the trajectory of a \textit{fractional Brownian motion} with a Hurst parameter of $\alpha = 1/d$. A formal definition is available in Appendix~\ref{a:COI}.

\begin{figure}[H]
    \centering
    \includegraphics[width=.75\linewidth]{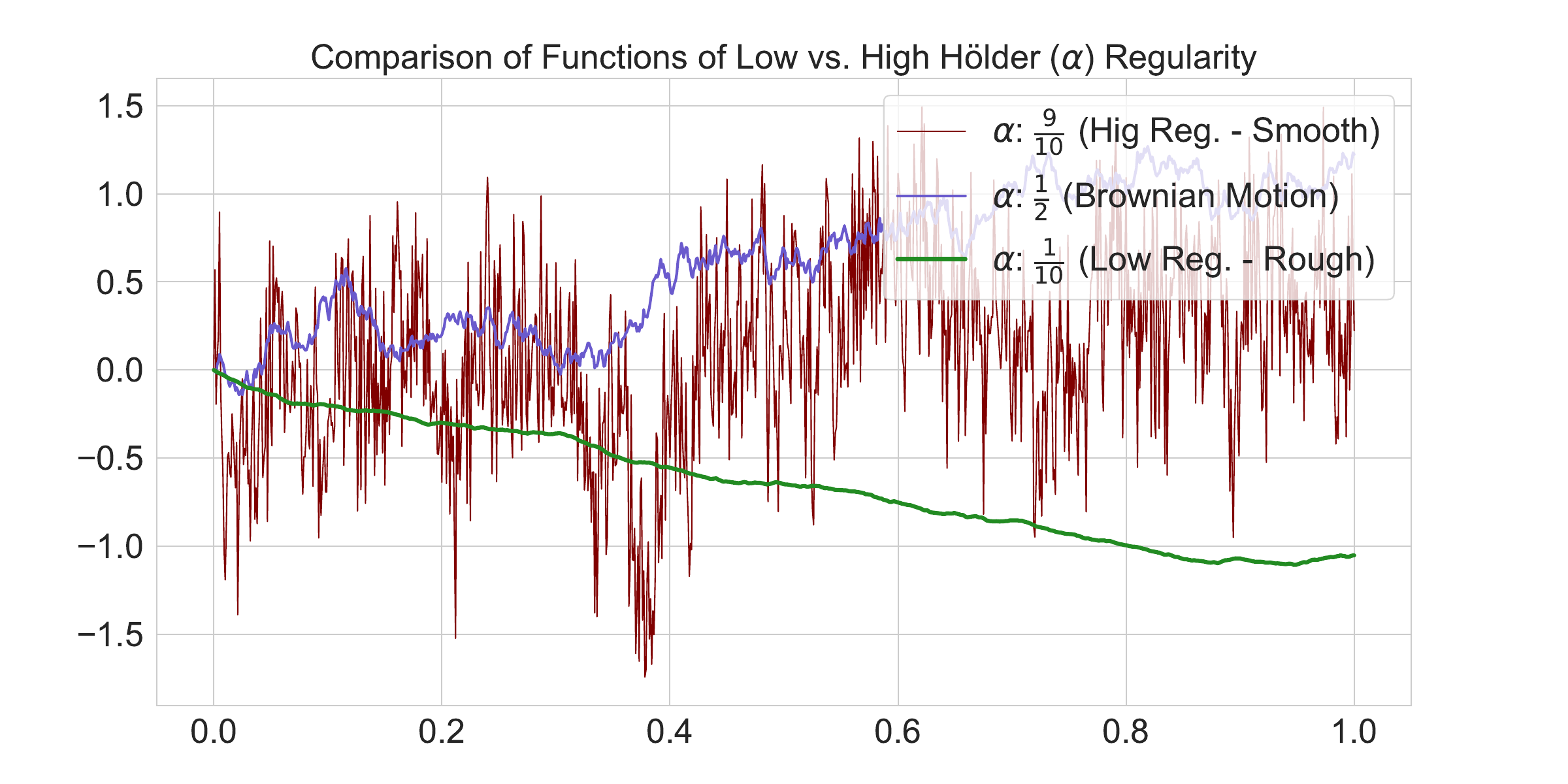} \vspace{-18pt}
    \caption{Visual Comparison of Functions with High ($\alpha\approx 1$) vs.\ Low ($\alpha\approx 0$) H\"{o}lder regularity.  If $\alpha\approx 1$, the function (green) is approximately differentiable almost everywhere, meaning it does not osculate much locally and thus is simple to approximate.  If $\alpha\approx 0$, the function may be nowhere differentiable and jagged; its extreme details make it difficult to approximate. }    
    \label{fig:Holder_lowavshigha}
\end{figure}

\begin{table}[H]
    \caption{Test mean squared error  (average and standard deviation) for the different functions of the regression task.}
    \label{tab:regression}
    \centering \footnotesize
    \scalebox{0.}{
    \begin{tabular}{l | c c c c c} \toprule
       & 1D H\"{o}lder  & 2D Ackley & 3D Ackley & 2D Rastrigin & 3D Rastrigin  \\ \midrule
       MoMLPs (ours) & \textbf{0.057} $\pm$ \textbf{0.085} & \textbf{0.00015} $\pm$ \textbf{0.00006} & \textbf{0.00068} $\pm$ \textbf{0.00010} & \textbf{0.0480} $\pm$ \textbf{0.0073} & \textbf{1.0062} $\pm$ \textbf{0.0446} \\ 
       Baseline &  0.128 $\pm$ 0.012 &  0.08723 $\pm$  0.01059 & 0.09303 $\pm$ 0.03156 & 3.0511 $\pm$ 0.3581 & 8.0376 $\pm$ 4.0499 \\ \bottomrule
    \end{tabular} }
\end{table}

\textbf{2D and 3D Functions.} We select the Ackley and Rastrigin functions, with their respective 2D representations showcased in Figure~\ref{fig:combined_plots}, as widely recognized benchmarks in the field of optimization.

\begin{figure}[htbp!]
    \centering
    \begin{subfigure}[b]{0.21\linewidth}
        \includegraphics[width=\linewidth]{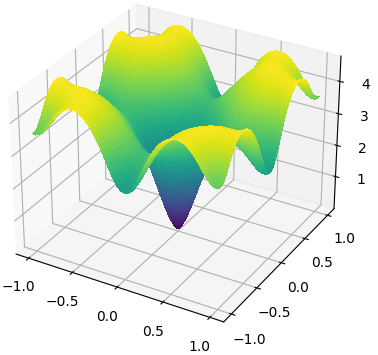}\vspace{-5pt}
        \caption{Ground Truth Ackley}
    \end{subfigure}
    \begin{subfigure}[b]{0.21\linewidth}
        \includegraphics[width=\linewidth]{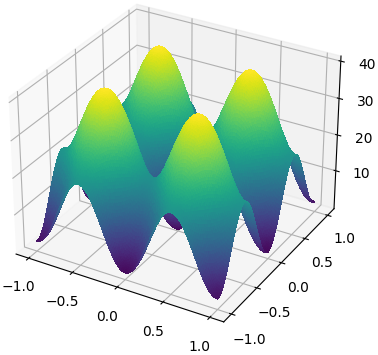}\vspace{-5pt}
        \caption{Ground Truth Rastrigin}
    \end{subfigure}

    \begin{subfigure}[b]{0.21\linewidth}
        \includegraphics[width=\linewidth]{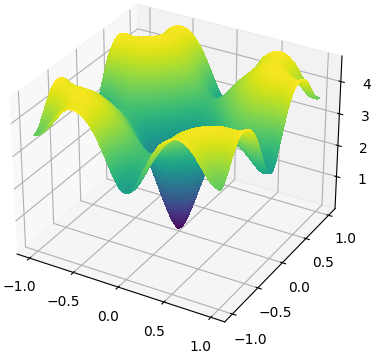} \vspace{-5pt}
        \caption{Predicted Ackley}
    \end{subfigure}
    \begin{subfigure}[b]{0.21\linewidth}
        \includegraphics[width=\linewidth]{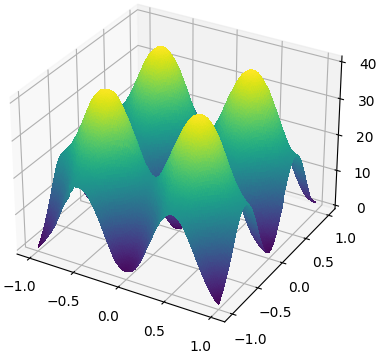} \vspace{-5pt} 
        \caption{Predicted Rastrigin}
    \end{subfigure}

    \begin{subfigure}[b]{0.21\linewidth}
        \includegraphics[width=\linewidth]{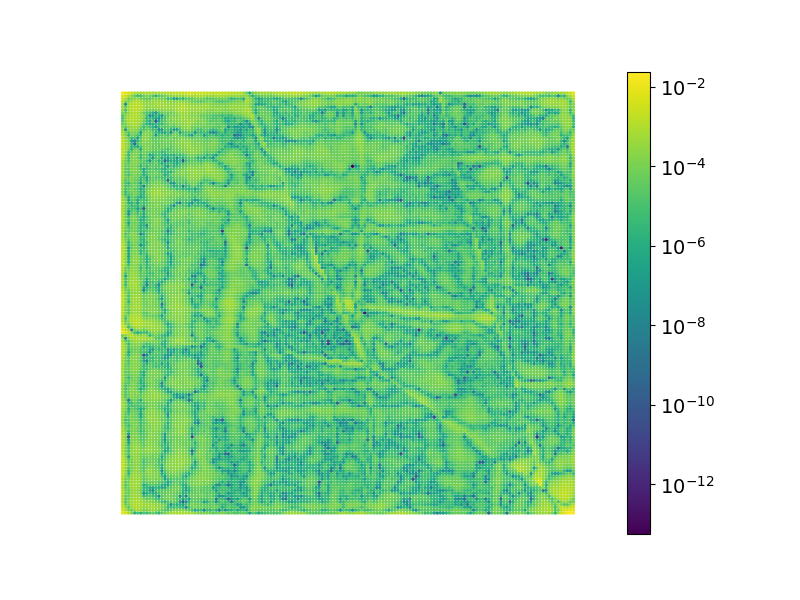}\vspace{-5pt} 
        \caption{log MSE}
    \end{subfigure}
    \begin{subfigure}[b]{0.21\linewidth}
        \includegraphics[width=\linewidth]{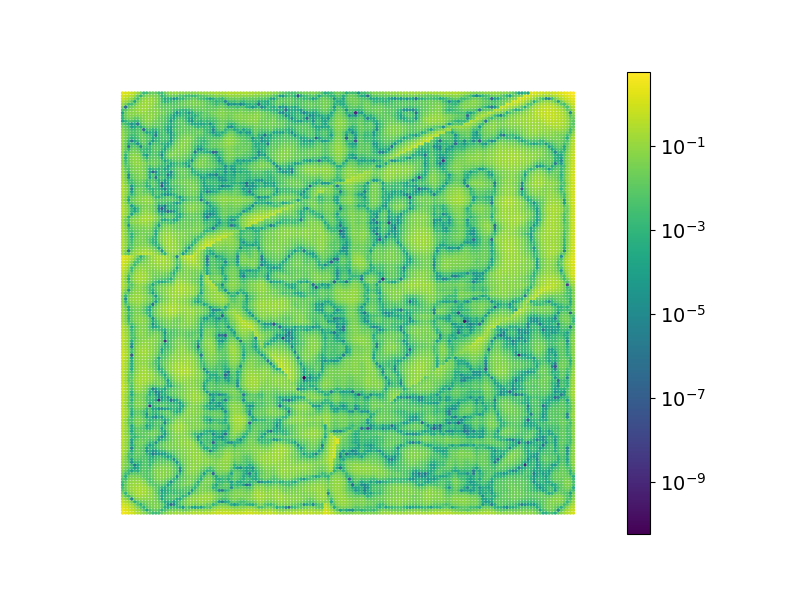}\vspace{- 5pt} 
        \caption{log MSE}
    \end{subfigure}
    \caption{Comparison of ground truth and predicted results for 2D Ackley and Rastrigin functions over the domain $[-1,1]^2$.}
    \label{fig:combined_plots}
    \vspace{-0.5cm}
\end{figure}

\textbf{Evaluation protocol.} We consider the setup where the domain of a function that we try to approximate is the $\inputdim$-dimensional closed set $[a,b]^\inputdim$. For instance, we arbitrarily choose the domain $[0,1]$ when $\inputdim = 1$, and $[-1,1]^\inputdim$ when $\inputdim \geq 2$. Our training and test samples are the $\samplesize^{\inputdim}$ vertices of the regular grid defined on $[a,b]^\inputdim$. 
At each run, $80\%$ of the samples are randomly selected for training and validation, and the remaining $20\%$ for testing.

During training, for a given and fixed set of $K$ prototypes $\prototype\eqdef(\prototype_1,\dots,\prototype_K)$, we assign each training sample $x$ to its nearest prototype $\prototype_k$ and associated neural network $\hatf_k$. We learn the prototypes as explained in Section \ref{s:Training}. For simplicity, we set the number of prototypes to $K = 4$; we also set $\samplesize = 10,000$ if $\inputdim = 1$, $\samplesize = 150$ if $\inputdim = 2$ (i.e., $150^2 = 22,500$ samples), and $\samplesize = 30$ if $\inputdim = 3$ (i.e., $27,000$ samples). More details can be found in Appendix \ref{sec:experimental_details}.

\textbf{Test performance.} In Table \ref{tab:regression}, we present the mean squared error obtained on the test set across 10 random initializations and various splits of the training/test sets. The baseline consists of a single neural network with the same overall architecture as each MoMLP but possesses as many parameters as all the MoMLPs combined. In all instances, the MoMLP model demonstrates a significant performance advantage compared to the baseline. Figure \ref{fig:combined_plots} illustrates the predictions generated by our MoMLPs, showcasing the capability of our approach to achieve a good approximation of the ground truth functions.

\vspace{-0.15cm}
\subsection{Classification}

\textbf{Datasets.} We evaluate classification on standard image datasets such as CIFAR-10 \citep{krizhevsky2010convolutional}, CIFAR-100, and Food-101 \citep{bossard2014food}, which consist of 10, 100, and 101 different classes, respectively. We use the standard splits of training/test sets: the datasets include (per category) 5,000 training and 1,000 test images for CIFAR-10, 500 training and 100 test images for CIFAR-100, and 750 and 250 for Food-101.

\textbf{Training.} Our MoMLP model takes as input latent DINOv2 encodings~\citep{oquab2023dinov2} of images from the aforementioned datasets. Each sample $x \in \mathbb{R}^{768}$ corresponds to a DINOv2 embedding (i.e., $\inputdim = 768$). Additionally, we set the prototypes as centroids obtained through the standard $K$-means clustering on the DINOv2 embedding space. The replacement of our original prototype learning algorithm is sensible in this context, as we operate within a structured latent space optimized through self-supervised learning using large-scale compute and datasets.
\begin{table}[!t]
    \caption{Test classification accuracy using DINOv2 features as input (average and standard deviation).}
    \label{tab:classification}
    \centering \scriptsize
    \begin{tabular}{l | c c c } \toprule
      Dataset & CIFAR-10  & CIFAR-100 & Food-101 \\ \midrule
       Ours (weighted) & 98.40 $\pm$ 0.05 & \textbf{90.01 $\pm$ 0.11} & \textbf{91.86 $\pm$ 0.10} \\ 
       Ours (unweighted) & 98.42 $\pm$ 0.04 & 89.62 $\pm$ 0.25 & 91.79 $\pm$ 0.16 \\ 
       Baseline & \textbf{98.45 $\pm$ 0.06} & 89.85 $\pm$ 0.17 & 91.45 $\pm$ 1.09 \\ \bottomrule
    \end{tabular} \vspace{-0.5cm}
\end{table}

Due to the potential class imbalance in the various Voronoi cells formed by the prototypes, we utilize two variations of the cross-entropy loss for each MoMLP $\hatf_k$. The first variation, termed \emph{unweighted}, assigns equal weight to all categories. The \emph{weighted} variation assigns a weight that is inversely proportional to the distribution of each category in the Voronoi cell defined by the prototype $\prototype_k$.

\textbf{Test performance.} We present the test classification accuracy over 10 random initializations of both the baseline and our MoMLPs in Table \ref{tab:classification}. The weighted version performs slightly better on datasets with a large number of categories. Our approach achieves comparable results with the baseline, effectively decomposing the prediction across multiple smaller models.

\subsection{Experimental details} \label{sec:experimental_details}

We include here experimental details, we refer to the source code in the supplementary material for more details.
We first outline the algorithm used to train the MoMLP MoE model.  We then provide details on the trained architecture and hyperparameter details in the implementation.

\subsection{Training Algorithm}
\label{s:Training}

We now provide an explanation for the training algorithm. As discussed, we mitigate down the algorithm into two parts: discovering the prototypes and training the networks. Conceptually, the prototypes define where in the input space the networks are located, or in other words, where in the input space we expect each of the networks to have the best performance. During inference, we will route a given input to the appropriate network based on its nearest prototype. In essence, each network learns to approximate a specific region of the overall input domain.

\textbf{Discovering prototypes.} In principle, we may not know how to partition the input space. One approach is to utilize standard clustering algorithms like $K$-means, but this might be suboptimal for the downstream task unless we are already operating in a structured latent space, such as those found in pre-trained models (further discussion on this is available in Section \ref{a:Experiments}). Another way is to learn it via gradient descent by optimizing the location of the prototypes for a specific task by following the gradient of the downstream loss. 
At the beginning of training, we have $\hat{F}(x)\eqdef (\hatf_1(x),\dots,\hatf_K(x))$ which contains a collection of $K$ randomly initialized shallow or small networks (i.e., much smaller than our MoMLPs described later). 
In this first step, we assume that we are able to load all randomly initialized networks into our GPU memory. In particular, this is true because we use small networks with few parameters, which we will later ``deepen'' in the next step by adding additional hidden layers. We initialize the prototypes $\prototype\eqdef(\prototype_1,\dots,\prototype_K)$ randomly from a uniform distribution within the bounds of our training dataset input samples. We use the following expression to train the location of our prototypes $\{\prototype_k\}_{k=1}^K$ in $\mathbb{R}^n$ by minimizing the energy:
\begin{equation}
           \sum_{(x,y)\in \mathcal{D}}\,
            \ell\big(
                    \operatorname{softmax}\big(
                        -\|x-\prototype_i\|_{i=1}^K
                    \big)^{\top}\, \hat{F}(x)\,
                ,
                    y
            \big)
. 
\end{equation}
where the loss $\ell$ is task-specific; for example, one could use mean squared error for regression and cross-entropy for classification. The $\operatorname{softmax}$ weights the importance of the prediction of each of the MoMLPs in $\hat{F}$ for a given input $x$, as a function of the input's distance to the prototypes, $\|x-\prototype_i\|_{i=1}^K \eqdef (\|x-\prototype_1\|,...,\|x-\prototype_K\|)$. 
Both the locations of the prototypes and the shallow randomly initialized neural networks assigned to them are optimized. 

\textbf{Deepening the Networks.} After the initial training phase, we enhance the networks by incorporating additional layers. Specifically, we introduce linear layers with weights initialized to the identity matrix and bias set to zero, just before the final output layer of each network. To encourage gradient flow in these new layers during the subsequent training stage, we slightly perturb this initialization with small Gaussian noise. This approach is driven by the fact that in the second training stage, each MoMLP can be optimized in a distributed manner. Consequently, we can work with larger networks without the need to load all of them simultaneously into our GPU, allowing for more model parameters. During the first stage, we have already optimized our networks alongside the prototype locations, converging towards a minimum. By initializing the networks with the new layers close to the identity, we can ensure that their output at the start of the second stage of training is similar to that produced by the original networks. This allows us to smoothly continue the optimization process from the point where we previously halted.

\textbf{MoMLP Training.} Once prototype locations have been fixed we can independently train MoMLPs $\hatf_1,\dots,\hatf_K$ by minimizing for all $ k \in \{ 1, \dots, K \}$:
\begin{equation}
\underset{\underset{k \in \argmin_{j \in \{ 1, \dots, K \}} \{ \| x - \prototype_j \| \}}{(x,y)\in \mathcal{D}}
                }{\sum}\,
                    \ell(\hatf_k(x),y)
\end{equation}
over all the networks. We optimize the MoMLP network $\hatf_k$ for training data points that are closest to prototype $\prototype_k$. The training procedure is summarized in Algorithm~\ref{alg:train_CNO}.


\begin{algorithm}[!t]
\caption{MoMLPs Training.} 
\label{alg:train_CNO}
\begin{algorithmic}
\SetAlgoLined
\Require Training data $\mathcal{D}\eqdef \{(x_j,y_j)\}_{j=1}^N$, no. of prototypes $K\in \mathbb{N}_+$, loss function $\ell$.
\DontPrintSemicolon

\State \textbf{Discovering Prototypes:}\\

\hspace{0.3cm} \resizebox{0.38\textwidth}{!}{{\small$\displaystyle (\hat{F}, \prototype) \leftarrow {\argmin_{\hat{F},\prototype}}\,
       \displaystyle \sum_{(x,y)\in \mathcal{D}}\,
            \ell\left(
                    \operatorname{softmax}\left(
                        x|\prototype
                    \right)^{\top}\hat{F}(x)
                ,
                    y
            \right)
    $}}

\State \textbf{Deepen networks:}\\
\hspace{0.3cm} \textbf{For $k=1,\dots,K$:}\\
\hspace{0.4cm} $\hatf_k \leftarrow \textrm{deepen}(\hatf_k)$

\State \textbf{MoMLP Training:}\\
\hspace{0.3cm}\textbf{For $k=1,\dots,K$:} \\
\hspace{0.4cm} $\hatf_k \leftarrow \underset{\hatf_k}{\argmin}\,
                \underset{\underset{k \in \argmin_{j} \{ \| x - \prototype_j \| \}}{(x,y)\in \mathcal{D}}}{\sum}\,
                    \ell(\hatf_k(x),y)$

\State \Return MoMLP parameters $\{\hatf_k\}_{k=1}^K$ and prototype locations $\{\prototype_k\}_{k=1}^K$.
\end{algorithmic}
\end{algorithm}

\textbf{Inference.} At inference time, each test sample $x$ is assigned to its nearest prototype $\prototype_k$ where $\displaystyle k \in \argmin_{j \in \{ 1, \dots, K \}} \{ \| x - \prototype_j \| \}$ and the prediction is made by the $k$-th MoMLP $\hatf_k$.

\textbf{Comparison to Standard Distributed Training.} 
One can distribute the complexity of feedforward models by storing each of their layers in offline memory and then loading them sequentially into VRAM during the forward pass.  
This does avoid loading more than $\mathcal{O}(\operatorname{Width})$ active parameters into VRAM at any given time, where $\operatorname{Width}$ denotes the width of the feedforward model.  However, doing so implies that all the model parameters are ultimately loaded during the forward pass.  This contrasts with the MoMLP model, which requires $\mathcal{O}\big(
\log_2(K)\, \operatorname{Width}^2\,\operatorname{Depth}\big)$ to be loaded into memory during a forward pass; where $\operatorname{Width}$ and $\operatorname{Depth}$ are respectively the largest width and depth of the MLP at any leaf of the tree defining a given MoMLP, and $K$ denotes the number of prototypes.  However, in the forward pass, one loads $\mathcal{O}(\varepsilon^{-n/2})$ parameters for the best worst-case MLP while only $\mathcal{O}(n\log(1/\varepsilon)/\varepsilon )$ are needed in the case of the MoMLPs. The number of parameters here represents the optimal worst-case rates for both models (see Table~\ref{tab:main_techincal} and Theorem~\ref{thm:Main__TechnicalVersion}).

\subsection{Definitions of the Ackley and Rastrigin functions}

Let us note $x = (x_1, \dots, x_{\inputdim})^{\top} \in \mathbb{R}^\inputdim$ the $\inputdim$-dimensional representation of a sample, we use the following formulation of the Ackley function: 
\begin{equation}
\operatorname{Ackley}(x) = 20 + \exp(1) -a \exp\left(-b \sqrt{\frac{1}{\inputdim}\sum_{i=1}^{\inputdim} x_i^2}\right) - \exp \left( \frac{1}{\inputdim}\sum_{i=1}^{\inputdim} \cos(2 \pi x_i) \right) 
\end{equation}
where $a = 20$ and $b = 0.2$. We also use the following formulation of the Rastrigin function:
\begin{equation}
\operatorname{Rastrigin}(x) = \sum_{i=1}^{\inputdim} x_i^2 + 10 \left(n - \sum_{i=1}^{\inputdim} \cos(2 \pi x_i) \right)
\end{equation}

\subsection{Architectures and hyperparameters}

\subsubsection{MoMLPs}

We consider that the width of our MoMLPs is $w = 1000$. In other words, each hidden layer of our MoMLPs contains a linear matrix of size $w \times w$.

In the regression task, our MoMLPs contain 3 hidden layers and we use a learnable PReLU as activation function. We use Adam \cite{kingma2014adam} as optimizer with a learning rate of $10^-4$ and other default parameters of Pytorch.

In the classification task, we follow the setup of \citet{oquab2023dinov2} and then use AdamW \cite{loshchilov2018decoupled} as optimizer with a learning rate of $10^{-3}$ and other default parameters of Pytorch, our MoMLPs contain 4 hidden layers in the classification task, and we apply BatchNorm1d before the PReLU activation function. 

\subsubsection{Baseline}

The baseline has the same architecture as the MoMLPs above. However, if we note $K$ the number of prototypes and assume that $\sqrt{K}$ is a natural number, the width of the baseline is $w \sqrt{K}$ so that the number of hidden parameters is the same as all the MoMLPs combined.

In our experiments, we set $K=4$, so $\sqrt{K}=2$.

\bibliographystyle{icml2024}
\bibliography{Bookkeeping/3_References,Bookkeeping/3_Referencesz_TBSorted}

\end{document}